\documentclass[11pt, twoside, sort, a4paper]{article}

\usepackage{url}
\usepackage{wrapfig}
\usepackage{enumitem}
\usepackage{graphicx}
\usepackage{caption}
\usepackage{subcaption}
\usepackage{hyperref}
\usepackage{bm}
\usepackage{mathtools}
\usepackage{float}
\usepackage{color}
\usepackage{footnote}
\makesavenoteenv{tabular}
\makesavenoteenv{table}
\usepackage{geometry}
\usepackage{empheq}
\usepackage{mathrsfs}
\usepackage{tikz}
\usepackage{cases}
\usepackage{amsmath,amsthm,amssymb,amsfonts}
\usepackage{booktabs}
\usepackage{multirow}
\usepackage{nccmath}
\usepackage{algorithm}
\usepackage{algorithmic}
\usepackage{comment}
\usepackage[titletoc,title]{appendix}
\usepackage[running,mathlines]{lineno}
\usepackage{titlesec}
\usepackage[normalem]{ulem}
\usepackage{cancel}

\usetikzlibrary{matrix,positioning,calc,arrows.meta}

\DeclareCaptionLabelFormat{lowercaseparens}{(#2)}
\definecolor{dpurple}{rgb}{0.6,0.00,0.50}

\newcommand{\dpurple}[1]{\color{black}{#1}}

\newcommand{\blue}[1]{{\color{black}{#1}}}
\newcommand{\sblue}[1]{{\color{black}{#1}}}

\AtBeginDocument{
  \addtolength{\abovedisplayskip}{-0.5ex}
  \addtolength{\abovedisplayshortskip}{-0.5ex}
  \addtolength{\belowdisplayskip}{-0.5ex}
  \addtolength{\belowdisplayshortskip}{-0.5ex}
}

\makeatletter
\def\thm@space@setup{\thm@preskip=3pt
\thm@postskip=3pt}
\makeatother

\titlespacing*{\section}
  {0pt}{0.6\baselineskip}{0.2\baselineskip}
\titlespacing*{\subsection}
  {0pt}{0.6\baselineskip}{0.2\baselineskip}
\titlespacing*{\subsubsection}
  {0pt}{0.6\baselineskip}{0.2\baselineskip}

\newcommand{\EQ}{\begin{equation}}
\newcommand{\EN}{\end{equation}}
\newcommand{\EQS}{\begin{equation*}}
\newcommand{\ENS}{\end{equation*}}
\newcommand{\EQA}{\begin{eqnarray}}
\newcommand{\ENA}{\end{eqnarray}}
\newcommand{\EQAS}{\begin{eqnarray*}}
\newcommand{\ENAS}{\end{eqnarray*}}

\newcommand{\Rbb}{\mathbb{R}}
\newcommand{\Nbb}{\mathbb{N}}

\newtheorem{theorem}{Theorem}
\newtheorem{lemma}{Lemma}
\newtheorem{remark}{Remark}
\newtheorem{assumption}{Assumption}

\newtheorem{corollary}{Corollary}

\numberwithin{equation}{section}
\numberwithin{table}{section}
\numberwithin{figure}{section}
\numberwithin{definition}{section}
\numberwithin{theorem}{section}
\numberwithin{lemma}{section}
\numberwithin{remark}{section}
\numberwithin{assumption}{section}
\numberwithin{condition}{section}
\numberwithin{property}{section}
\numberwithin{proposition}{section}
\numberwithin{corollary}{section}
\numberwithin{example}{section}
\numberwithin{algorithm}{section}

\begin{document}

\title{A data-driven Fourier-mixture neural-network method for density estimation}

\author{
Duy-Minh Dang\thanks{School of Mathematics and Physics, The University of Queensland,
St Lucia, Brisbane 4072, Australia,
email: \texttt{duyminh.dang@uq.edu.au}}
\and
Volter Entoma\thanks{School of Mathematics and Physics, The University of Queensland,
St Lucia, Brisbane 4072, Australia,
email: \texttt{v.entoma@student.uq.edu.au}}
}

\date{}
\maketitle

\begin{abstract}
We propose a data-driven Fourier-trained neural-network method for estimating
fixed-horizon probability densities from empirical characteristic-function
(CF) information. The estimator is a positive Gaussian--Laplace mixture with
closed-form CF, so training can be performed directly in Fourier space while
preserving nonnegativity and unit mass. We consider two sampling settings. In
the direct i.i.d.\ sampling setting, the method is trained against an empirical
CF constructed from i.i.d.\ samples. In the resampling-based pseudo-sampling
setting, it is trained against an empirical pseudo-CF constructed from
dependent data by resampling. For the direct i.i.d.\ case, we derive an
expected squared $L_2$ density-error bound that separates Fourier
truncation, empirical training error, discretization, and CF sampling error,
together with a corresponding bound for the expected $L_2$ norm.
For the pseudo-sampling case, we obtain conditional counterparts with
an additional pseudo-law discrepancy term. We develop a multidimensional
extension of the framework and analyze its computational complexity.
Numerical experiments show competitive performance relative to Expectation--Maximization on Gaussian-mixture benchmarks and clear gains on heavy-tailed targets. A controlled comparison with exact-CF training quantifies the additional density error introduced by empirical-CF training; further experiments show that the $L_2$ norm decays approximately as the inverse square root of the sample size when the remaining error terms are controlled, and demonstrate the method on a one-year Australian equity return law estimated from resampled dependent data.

\vspace{.1in}


\vspace{.1in}

\end{abstract}
\section{Introduction}
\label{sec:intro}
Estimating a fixed-horizon probability law from data is a basic task in computational science. Such laws arise whenever one seeks to approximate the distribution of a state increment over a prescribed time interval, for example in stochastic simulation, uncertainty quantification, filtering, and data-driven decision problems. In many applications, the target density is unavailable in closed form and only sample information is available. This is especially true when the underlying observations form a dependent time series and the quantity of interest is a law over a longer horizon than the sampling frequency of the data.

Two broad classes of density estimation methods are commonly used in this setting.
The first consists of physical-space estimators, such as kernel density estimation \cite{silverman1986density} and finite-mixture models fitted by likelihood-based procedures \cite{mclachlan2000finite}. Among these, Gaussian mixtures estimated by Expectation--Maximization (EM) provide a classical positive parametric benchmark \cite{dempster1977maximum,mclachlan2000finite}.

The second class uses Fourier information, exploiting the fact that a characteristic function (CF) always exists and may be easier to analyze or approximate than the density itself. A prominent example is the COS method \cite{fang2008novel}, which approximates densities by finite cosine-series expansions with expansion coefficients obtained from the CF. Data-driven variants have also been developed; in particular, the data-driven COS method of \cite{leitao2018data} uses samples to construct a Fourier-cosine density estimator without requiring a closed-form CF. However, truncated Fourier-cosine expansions do not in general enforce nonnegativity of the recovered density, and this loss of positivity can be problematic for short-horizon transition densities and for subsequent stochastic-control computations \cite{du2025fourier,ForsythLabahn2017}. Fourier-trained neural-network (NN) methods provide another Fourier-domain route in the exact-transform setting; see, for example, \cite{du2025fourier,dang2026monotone}, where the target CF is assumed known in closed form.

The present paper focuses on the data-driven Fourier-trained NN setting, where the exact CF is unavailable but empirical transform information can still be constructed from samples. One may have access to direct i.i.d.\ samples from the target fixed-horizon law, from which an empirical CF is constructed. In the presence of temporal dependence, one instead observes a dependent time series, from which an empirical pseudo-CF is constructed by a suitable resampling procedure. This motivates a Fourier-trained NN model that remains a proper density throughout optimization, is expressive enough to capture non-Gaussian features, and admits a closed-form CF so that the learning problem can be posed directly in Fourier space.

The main contributions of the paper are as follows.
\begin{itemize}[noitemsep, topsep=2pt, leftmargin=*]
\item
We introduce a data-driven Fourier-trained NN method based on a positive Gaussian--Laplace mixture parametrized by a bounded single-hidden-layer feedforward NN. The resulting density is nonnegative and integrates to one by construction, while its CF is available in closed form. This allows the learning problem to be posed directly against empirical transform data.

\item
For the direct i.i.d.\ sampling setting, we derive an {\dpurple{expected squared $L_2$ density-error bound, together with a corresponding bound for the expected $L_2$ norm}}. A central technical ingredient is a Fourier-to-real-space analysis, via Plancherel's theorem, that transfers control of the Fourier-domain training error to control of the density error in physical space. The resulting decomposition separates Fourier truncation, empirical training error, discretization error, and empirical-CF sampling error.

\item
For the resampling-based pseudo-sampling setting, we show that the same analysis carries over conditionally on the observed history. The resulting {\dpurple{conditional error bounds have the same structural form as in the direct i.i.d.\ case, with an additional squared pseudo-law discrepancy term in the squared-$L_2$ density-error bound}}. This provides a theoretical basis for Fourier-trained NN estimation from dependent time-series data.

\item
We extend the method to multiple dimensions, analyze its computational complexity, and validate it numerically on a range of benchmark and data-driven examples. These include {\dpurple{a controlled comparison of exact-CF reference and empirical-CF training,}} Gaussian-mixture sanity checks, heavy-tailed and jump benchmarks, a two-dimensional (2D) Cauchy example, {\dpurple{an $L_2$ error-decay study}}, and a resampling-based pseudo-sampling experiment built from Australian equity data.
\end{itemize}
The paper therefore combines several features that are rarely treated together: data-driven Fourier training of a single-hidden-layer NN, a positivity-preserving Gaussian--Laplace output model, Fourier-to-real-space analysis underpinning rigorous density error bounds, and a treatment of both direct i.i.d.\ sampling and dependent-data pseudo-sampling. The numerical results support this picture: the controlled comparison quantifies the additional density error introduced by empirical-CF training and shows that it decreases as the sample size grows; the method is competitive with EM on Gaussian-mixture benchmarks, improves over the Gaussian subclass on heavy-tailed targets, exhibits $L_2$-norm decay close to the inverse-square-root rate in the sample size in a well-specified setting when the remaining error terms are controlled, and remains effective in the resampling-based Australian equity experiment.

Although some numerical illustrations use financial return data, the methodology is not finance-specific. The central problem is more general: estimating a positive fixed-horizon density from empirical transform information, including settings where the available observations are dependent rather than i.i.d. In this sense, the present work can be viewed as a data-driven Fourier counterpart to positive finite-mixture density estimation, with emphasis on Fourier-trained NN models, transform-domain learning, and rigorous error analysis.

The remainder of the paper is organized as follows. Section~\ref{sc:NN_Fourier} introduces the neural network Fourier-mixture method in the direct i.i.d.\ sampling setting. Section~\ref{sec:error} presents the corresponding error analysis. Section~\ref{sec:bootstrap} extends the framework to resampling-based pseudo-sampling from dependent data. Section~\ref{sc:multi} gives the multi-dimensional extension, and Section~\ref{sc:complex} discusses computational complexity. Numerical experiments are reported in Section~\ref{sec:num}. Finally, Section~\ref{sec:conclusion} provides concluding remarks and outlines future directions.

\section{A neural network Fourier-mixture method}
\label{sc:NN_Fourier}
\subsection{Data-driven setting}
We work in a data-driven setting. Given observed data points $X_1,\dots,X_M\in\Rbb$, we seek to estimate an unknown probability density $g$ on $\Rbb$ from Fourier-domain information. Throughout, we use the Fourier pair
\begin{equation}
\label{eq:G_M}
G(\eta)
=
\int_{\Rbb} e^{i\eta x}\,g(x)\,dx,
\quad \eta\in\Rbb,
\qquad
g(x)
=
\frac{1}{2\pi}\int_{\Rbb} e^{-i\eta x}\,G(\eta)\,d\eta,
\quad x\in\Rbb.
\end{equation}
In contrast with many Fourier-based approaches, including \cite{du2025fourier}, where the CF $G(\cdot)$ is available in closed form, here we work only with the empirical CF
given by
\begin{equation}
\label{eq:hatG_M}
\widehat G_M(\eta)
=
\frac{1}{M}\sum_{m=1}^{M} e^{i\eta X_m},
\qquad \eta\in\Rbb.
\end{equation}
In this section, our objective is to recover $g$ from $\widehat G_M$ by fitting a bounded class of Gaussian--Laplace mixtures parametrized by a single-hidden-layer FFNN, whose Fourier transforms are available in closed form.
\subsection{Description}
We first state the regularity conditions used in the subsequent error analysis.

\begin{assumption}[Modelling assumptions]
\label{ass:modelling}
Let $g$ denote the target density, and let $G$ denote its Fourier transform, i.e.\ the CF of the associated probability law. The following conditions hold.
\begin{enumerate}[noitemsep, topsep=2pt,label=(A\arabic*), leftmargin=*]
\item The observations $X_1,\dots,X_M$ are i.i.d.\ with common density $g$.
\item The density satisfies $g\in L_1(\Rbb)\cap L_\infty(\Rbb)$.
\item For any finite interval $[\eta_1,\eta_2]\subset\Rbb$ with $\eta_1<\eta_2$, the CF $G$ is Lipschitz continuous on $[\eta_1,\eta_2]$, i.e.\ there exists a finite constant $L_{[\eta_1,\eta_2]}>0$ such that
\[
|G(\eta')-G(\eta'')|
\le
L_{[\eta_1,\eta_2]}\,|\eta'-\eta''|,
\qquad
\forall\,\eta',\eta''\in[\eta_1,\eta_2].
\]
\end{enumerate}
\end{assumption}

\begin{remark}[Discussion of Assumption~\ref{ass:modelling}]
\label{rem:assumptions_section2}
Assumption~\ref{ass:modelling}~(A1) is a sampling condition. Under (A1), for each fixed $\eta\in\Rbb$, the empirical CF in \eqref{eq:hatG_M} satisfies
\begin{equation}
\label{eq:ecf_mean_var_section2}
\mathbb{E}\!\left[\widehat G_M(\eta)\right]=G(\eta),
\qquad
\operatorname{Var}\!\left(\widehat G_M(\eta)\right)
=
\frac{1-|G(\eta)|^2}{M}.
\end{equation}
Hence $\widehat G_M$ is an unbiased estimator of the true CF $G$.
Dependent time-series data are addressed in Section~\ref{sec:bootstrap} through a resampling-based pseudo-sampling framework. As a first step, we work here in the direct i.i.d.\ setting to isolate the Fourier-inversion and approximation errors and to establish the baseline analysis for that extension.

Assumption~\ref{ass:modelling}~(A2) implies $g\in L_2(\Rbb)$, since $\|g\|_{L_1}=1$ and
\[
\|g\|_{L_2}^2
=
\int_{\Rbb}|g(x)|^2\,dx
\le
\|g\|_{L_\infty}\|g\|_{L_1}
=
\|g\|_{L_\infty}<\infty.
\]
By Plancherel's theorem, this yields $G\in L_2(\Rbb)$, and therefore for every $\varepsilon>0$ there exists $\eta'>0$ such that $\int_{\Rbb\setminus[-\eta',\eta']}|G(\eta)|^2\,d\eta\le \varepsilon$.

Assumption~\ref{ass:modelling}~(A3) is a local regularity condition. It requires Lipschitz continuity only on the compact Fourier interval used for training and discretization bounds. Assumptions~(A2) and~(A3) are mild for many smooth parametric families, including Gaussian laws; jump--diffusion models with a nondegenerate Gaussian component, such as the Merton model \cite{merton1976option} and the Kou model \cite{kou01}; and many tempered-tail L'evy models with closed-form CFs and bounded densities, including the Normal--Inverse--Gaussian model \cite{barndorff1997normal} and CGMY/tempered-stable families with $Y\in(0,2)$ under standard parameter regimes \cite{carr2002fine}.

\end{remark}

\paragraph{Gaussian--Laplace mixture representation.}
Let the Gaussian and Laplace activation functions be defined, respectively, by
\[
\phi_{\mathcal G}(z)=e^{-z^2},
\qquad
\phi_{\mathcal L}(z)=e^{-|z|}.
\]
We define the class of mixed single-layer FFNN
\begin{equation}
\label{eq:Sigma_GL_raw}
\Sigma(\phi_{\mathcal G},\phi_{\mathcal L})
=
\Big\{\widehat g:\Rbb\to\Rbb\ \big|\
\widehat g(x;\vartheta)
=
\sum_{k=1}^{K_{\mathcal G}}\widetilde\rho_k^{(\mathcal G)}
\phi_{\mathcal G}\big(\widetilde w_k^{(\mathcal G)}x+\widetilde b_k^{(\mathcal G)}\big)
+
\sum_{\ell=1}^{K_{\mathcal L}}\widetilde\rho_\ell^{(\mathcal L)}
\phi_{\mathcal L}\big(\widetilde w_\ell^{(\mathcal L)}x+\widetilde b_\ell^{(\mathcal L)}\big)
\Big\},
\end{equation}
where $K_{\mathcal G},K_{\mathcal L}\in\Nbb$ and $\vartheta$ collects the raw coefficients, weights, and biases.

For density estimation, we work instead with a positive normalized subclass parameterized by mixture logits, locations, and raw scales.
To this end, we define the normalized Gaussian and Laplace kernels
\begin{equation}
\label{eq:normalized_kernels_section2}
\varphi_{\mathcal G}(x;\mu,\sigma)
=
\frac{1}{\sqrt{2\pi}\sigma}
\exp\!\left(-\frac{(x-\mu)^2}{2\sigma^2}\right),
\qquad
\varphi_{\mathcal L}(x;\nu,b)
=
\frac{1}{2b}\exp\!\left(-\frac{|x-\nu|}{b}\right),
\end{equation}
with $\sigma>0$ and $b>0$. The raw hidden-unit form and the normalized kernels are related by
\begin{equation}
\label{eq:unit_mappings_section2}
\begin{aligned}
\mu_k
&=
-\frac{\widetilde b_k^{(\mathcal G)}}{\widetilde w_k^{(\mathcal G)}},
\qquad
\sigma_k
=
\frac{1}{\sqrt{2}\,|\widetilde w_k^{(\mathcal G)}|},
\qquad
\widetilde\beta_k^{(\mathcal G)}
=
\frac{\sqrt{\pi}\,\widetilde\rho_k^{(\mathcal G)}}{|\widetilde w_k^{(\mathcal G)}|},
\\
\nu_\ell
&=
-\frac{\widetilde b_\ell^{(\mathcal L)}}{\widetilde w_\ell^{(\mathcal L)}},
\qquad
b_\ell
=
\frac{1}{|\widetilde w_\ell^{(\mathcal L)}|},
\qquad
\widetilde\beta_\ell^{(\mathcal L)}
=
\frac{2\,\widetilde\rho_\ell^{(\mathcal L)}}{|\widetilde w_\ell^{(\mathcal L)}|}.
\end{aligned}
\end{equation}
We now introduce the positive normalized parametrization used for training. For $k=1,\ldots,K_\mathcal{G}$ and $\ell=1,\ldots,K_\mathcal{L}$, the trainable variables are the Gaussian mixture logits $\rho_k^{(\mathcal G)}$, locations $\mu_k$, and raw scales $w_k^{(\mathcal G)}$, together with the Laplace mixture logits $\rho_\ell^{(\mathcal L)}$, locations $\nu_\ell$, and raw scales $w_\ell^{(\mathcal L)}$. The mixture weights are obtained from the joint softmax map
\begin{equation}
\label{eq:softmax_weights_section2}
\begin{aligned}
\beta_k^{(\mathcal G)}
&=
\exp(\rho_k^{(\mathcal G)})
\bigg(
\sum_{j=1}^{K_{\mathcal G}}\exp(\rho_j^{(\mathcal G)})
+
\sum_{m=1}^{K_{\mathcal L}}\exp(\rho_m^{(\mathcal L)})
\bigg)^{-1},
\\
\beta_\ell^{(\mathcal L)}
&=
\exp(\rho_\ell^{(\mathcal L)})
\bigg(
\sum_{j=1}^{K_{\mathcal G}}\exp(\rho_j^{(\mathcal G)})
+
\sum_{m=1}^{K_{\mathcal L}}\exp(\rho_m^{(\mathcal L)})
\bigg)^{-1}.
\end{aligned}
\end{equation}
The Gaussian and Laplace scales are generated by the softplus map
\begin{equation}
\label{eq:softplus_scales_section2}
\sigma_k=\operatorname{softplus}\big(w_k^{(\mathcal G)}\big)+\varepsilon_{\mathrm{sf}},
\qquad
b_\ell=\operatorname{softplus}\big(w_\ell^{(\mathcal L)}\big)+
\varepsilon_{\mathrm{sf}},
\end{equation}
for $k=1,\ldots,K_{\mathcal G}$ and $\ell=1,\ldots,K_{\mathcal L}$, with fixed $\varepsilon_{\mathrm{sf}}>0$. The resulting Gaussian--Laplace mixture class is
\begin{equation}
\label{eq:Sigma_GL_pos}
\Sigma_{\mathcal G\mathcal L}
=
\Big\{\widehat g:\Rbb\to\Rbb\ \Big|\
\widehat g(x;\theta)
=
\sum_{k=1}^{K_{\mathcal G}}\beta_k^{(\mathcal G)}\varphi_{\mathcal G}(x;\mu_k,\sigma_k)
+
\sum_{\ell=1}^{K_{\mathcal L}}\beta_\ell^{(\mathcal L)}\varphi_{\mathcal L}(x;\nu_\ell,b_\ell),
\ \theta\in\Theta
\Big\}.
\end{equation}
The effective parameter vector is
\begin{equation}
\label{eq:theta}
\theta
=
\Big(
\{\beta_k^{(\mathcal G)},\mu_k,\sigma_k\}_{k=1}^{K_{\mathcal G}},
\{\beta_\ell^{(\mathcal L)},\nu_\ell,b_\ell\}_{\ell=1}^{K_{\mathcal L}}
\Big).
\end{equation}
We restrict the effective mixture parameters to a bounded set
\begin{align}
\label{eq:Theta_GL_section2}
\Theta
=
\Big\{\theta:\ &
\beta_k^{(\mathcal G)}\ge 0,\ \beta_\ell^{(\mathcal L)}\ge 0,
\quad
\sum_{k=1}^{K_{\mathcal G}}\beta_k^{(\mathcal G)}
+
\sum_{\ell=1}^{K_{\mathcal L}}\beta_\ell^{(\mathcal L)}
=
1,
\nonumber
\\
&
|\mu_k|,\,|\nu_\ell|\le \overline\mu,
\quad
0<\sigma_{\min}\le \sigma_k\le \sigma_{\max},
\quad
0<b_{\min}\le b_\ell\le b_{\max}
\Big\}.
\end{align}
Under \eqref{eq:Theta_GL_section2}, each $\widehat g(\cdot;\theta)\in\Sigma_{\mathcal G\mathcal L}$ is a probability density.

\paragraph{Closed-form characteristic function.}
Direct calculation gives
\begin{equation}
\label{eq:GL_cf_section2}
\widehat G(\eta;\theta)
=
\sum_{k=1}^{K_{\mathcal G}}
\beta_k^{(\mathcal G)}
\exp\!\left(i\eta\mu_k-\sigma_k^2\eta^2/2\right)
+
\sum_{\ell=1}^{K_{\mathcal L}}
\beta_\ell^{(\mathcal L)}
\frac{\exp(i\eta\nu_\ell)}{1+b_\ell^2\eta^2}.
\end{equation}
Hence
\begin{align}
\label{eq:GL_cf_real_section2}
\operatorname{Re}\widehat G(\eta;\theta)
&=
\sum_{k=1}^{K_{\mathcal G}}
\beta_k^{(\mathcal G)}
\exp\!\left(-\sigma_k^2\eta^2/2\right)\cos(\eta\mu_k)
+
\sum_{\ell=1}^{K_{\mathcal L}}
\beta_\ell^{(\mathcal L)}
\frac{\cos(\eta\nu_\ell)}{1+b_\ell^2\eta^2},
\\
\label{eq:GL_cf_imag_section2}
\operatorname{Im}\widehat G(\eta;\theta)
&=
\sum_{k=1}^{K_{\mathcal G}}
\beta_k^{(\mathcal G)}
\exp\!\left(-\sigma_k^2\eta^2/2\right)\sin(\eta\mu_k)
+
\sum_{\ell=1}^{K_{\mathcal L}}
\beta_\ell^{(\mathcal L)}
\frac{\sin(\eta\nu_\ell)}{1+b_\ell^2\eta^2}.
\end{align}
In particular, $\widehat G(0;\theta)=1$.

\paragraph{$L_2$-approximation error and Fourier-domain invariance.}
Since $g\in L_1(\Rbb)\cap L_2(\Rbb)$ and every $\widehat g(\cdot;\theta)\in\Sigma_{\mathcal G\mathcal L}$ with $\theta\in\Theta$ also belongs to $L_1(\Rbb)\cap L_2(\Rbb)$, Plancherel's theorem gives
\begin{equation}
\label{eq:GL_plancherel_section2}
\int_{\Rbb}\big|g(x)-\widehat g(x;\theta)\big|^2\,dx
=
\frac{1}{2\pi}\int_{\Rbb}\big|G(\eta)-\widehat G(\eta;\theta)\big|^2\,d\eta.
\end{equation}
This identity is the key structural fact used below: it allows both training and error analysis to be carried out through $\widehat G_M$ and $\widehat G(\cdot;\theta)$ rather than directly in real space.


\subsection{Training loss and regularization}
\label{subsec:training_loss_section2}
The exact CF $G$ is not available. Instead, training is based on $\widehat G_M$. We restrict the Fourier domain to $[-\eta',\eta']$ and choose training nodes $\{\eta_p\}_{p=1}^{P}\subset[-\eta',\eta']$ induced by a deterministic, potentially non-uniform, partition
\[
-\eta'=\xi_0<\xi_1<\cdots<\xi_P=\eta'.
\]
For each cell $[\xi_{p-1},\xi_p]$, let $\eta_p\in[\xi_{p-1},\xi_p]$ be the training node and set
\[
\delta_p=\xi_p-\xi_{p-1},
\qquad p=1,\ldots,P.
\]
We assume
\begin{equation}
\label{eq:eta_partition_section2}
\delta_{\min}=\frac{C_0}{P},
\qquad
\delta_{\max}=\frac{C_1}{P},
\qquad
\delta_{\min}=\min_{1\le p\le P}\delta_p,
\quad
\delta_{\max}=\max_{1\le p\le P}\delta_p,
\end{equation}
for finite constants $C_0,C_1>0$ independent of $P$.

For each training node $\eta_p$, define the empirical-CF real and imaginary residuals
\begin{equation}
\label{eq:residuals_section2}
\Delta_{p,\operatorname{Re}}(\theta)
=
\operatorname{Re}\widehat G_M(\eta_p)-\operatorname{Re}\widehat G(\eta_p;\theta),
\quad
\Delta_{p,\operatorname{Im}}(\theta)
=
\operatorname{Im}\widehat G_M(\eta_p)-\operatorname{Im}\widehat G(\eta_p;\theta).
\end{equation}
For $\lambda \ge 0$, we define the Fourier-domain loss function
\begin{equation}
\label{eq:Loss_MP_section2}
\mathrm{Loss}_{M,P}(\theta)
=
\mathrm{MSE}_{M,P}(\theta)
+
\lambda\,R_{M,P}(\theta),
\qquad
\theta\in\widehat\Theta\subseteq\Theta,
\end{equation}
where
\begin{equation}
\label{eq:MSE_MP_section2}
\begin{aligned}
\mathrm{MSE}_{M,P}(\theta)
&=
\frac{1}{P}\sum_{p=1}^{P}
\Big(
\Delta_{p,\operatorname{Re}}(\theta)^2
+
\Delta_{p,\operatorname{Im}}(\theta)^2
\Big),
\\
R_{M,P}(\theta)
&=
\frac{1}{P}\sum_{p=1}^{P}
\Big(
|\Delta_{p,\operatorname{Re}}(\theta)|
+
|\Delta_{p,\operatorname{Im}}(\theta)|
\Big).
\end{aligned}
\end{equation}
The first term in \eqref{eq:Loss_MP_section2} is the discrete Fourier-space $L_2$ misfit between the learned CF $\widehat G(\cdot;\theta)$ and the empirical CF $\widehat G_M$. The second term is an MAE contribution that aims to improve robustness in oscillatory or high-curvature frequency regions.

\paragraph{Data-driven  minimizer.}
Let $\widehat\Theta\subseteq\Theta$ denote the feasible search set induced by the chosen component counts $(K_{\mathcal G},K_{\mathcal L})$ and the corresponding bounds on the effective parameters. The data-driven minimizer is
\begin{equation}
\label{eq:theta_hat_section2}
\widehat\theta
\in
\arg\min_{\theta\in\widehat\Theta}\mathrm{Loss}_{M,P}(\theta),
\end{equation}
and the estimated density is $\widehat g(\cdot;\widehat\theta)\in\Sigma_{\mathcal G\mathcal L}$.

\paragraph{Optional affine preprocessing.}
When the empirical CF is highly oscillatory or supported over a large frequency range, it is often convenient to apply an affine  change of variables $Y=aX+c$ with $a>0$ before training \cite{du2025fourier}. The empirical CF of the transformed variable then satisfies $\widehat G_M^{Y}(\eta)=e^{i\eta c}\widehat G_M^{X}(a\eta)$.
After training in the $Y$-variable, the estimated density in the original variable is recovered by $\widehat g_X(x;\widehat\theta)=a\,\widehat g_Y(ax+c;\widehat\theta)$.

\paragraph{Training algorithm.}
We use a two-stage optimization strategy.
The first stage performs a coarse exploration of the loss landscape, while the second stage refines the solution using a smaller learning rate. The use of AMSGrad and Adam in these two stages follows standard NN practice; see \cite{AMSgrad, adam}.
\begin{algorithm}[H]
\caption{Data-driven Fourier-trained Gaussian--Laplace mixture model}
\label{alg:data_driven_fournet}
\begin{algorithmic}[1]
\STATE \textbf{Input:} data $\{X_m\}_{m=1}^{M}$, component counts $(K_{\mathcal G},K_{\mathcal L})$, Fourier window $[-\eta',\eta']$, training nodes $\{\eta_p\}_{p=1}^{P}$, and MAE weight $\lambda$.
\STATE If needed, apply affine preprocessing $Y_m=aX_m+c$ and replace the data by $\{Y_m\}_{m=1}^{M}$.
\STATE Compute $\widehat G_M(\eta_p)=\frac{1}{M}\sum_{m=1}^{M}e^{i\eta_p Y_m}$ for $p=1,\ldots,P$.
\STATE Initialize the trainable logits, locations, and raw scales.
\STATE \textbf{Stage I:} minimize $\mathrm{Loss}_{M,P}$ by AMSGrad (or Adam) with a relatively large learning rate.
\STATE \textbf{Stage II:} continue minimization with Adam at a smaller learning rate until convergence, or stop by early stopping / validation loss.
\STATE Return $\widehat\theta$, the estimated density $\widehat g(\cdot;\widehat\theta)$, and the learned CF $\widehat G(\cdot;\widehat\theta)$.
\end{algorithmic}
\end{algorithm}

\section{Error analysis}
\label{sec:error}
This section analyzes the expected {\dpurple{squared}} $L_2$ {\dpurple{density}} error under direct i.i.d.\ sampling. We first introduce the exact Fourier-domain loss on the training nodes, then quantify the empirical-CF sampling error, and finally combine this with discretization and truncation estimates.  {\dpurple{All expectations are taken with respect to the i.i.d.\ sample and, when applicable, any randomness in the training procedure.}}

For each training node $\eta_p$, define the exact real and imaginary residuals
\begin{equation}
\label{eq:exact_res}
\Delta_{p,\operatorname{Re}}^{*}(\theta)
=
\operatorname{Re}G(\eta_p)-\operatorname{Re}\widehat G(\eta_p;\theta),
\qquad
\Delta_{p,\operatorname{Im}}^{*}(\theta)
=
\operatorname{Im}G(\eta_p)-\operatorname{Im}\widehat G(\eta_p;\theta).
\end{equation}
Let
\begin{equation}
\label{eq:MSE_exact_section3}
\begin{aligned}
\mathrm{MSE}_{P}^{*}(\theta)
&=
\frac{1}{P}\sum_{p=1}^{P}
\Big(
\Delta_{p,\operatorname{Re}}^{*}(\theta)^2
+
\Delta_{p,\operatorname{Im}}^{*}(\theta)^2
\Big),
\\
R_{P}^{*}(\theta)
&=
\frac{1}{P}\sum_{p=1}^{P}
\Big(
|\Delta_{p,\operatorname{Re}}^{*}(\theta)|
+
|\Delta_{p,\operatorname{Im}}^{*}(\theta)|
\Big),
\end{aligned}
\end{equation}
and define the exact loss
\begin{equation}
\label{eq:Loss_exact_section3}
\mathrm{Loss}_{P}^{*}(\theta)
=
\mathrm{MSE}_{P}^{*}(\theta)
+
\lambda\,R_{P}^{*}(\theta),
\qquad
\theta\in\widehat\Theta\subseteq\Theta.
\end{equation}

\subsection{Empirical-CF sampling error}
We first quantify how finite-sample error in the empirical-CF propagates into the Fourier-domain loss evaluated on the \mbox{training nodes}.
\begin{lemma}[Empirical-CF sampling error]
\label{lem:ecf_error_section3}
For the training nodes $\{\eta_p\}_{p=1}^{P}$, define
\begin{equation}
\label{eq:VWP_section3}
V_P
=
\frac{1}{P}\sum_{p=1}^{P}\big(1-|G(\eta_p)|^2\big),
\qquad
W_P
=
\frac{1}{P}\sum_{p=1}^{P}\sqrt{1-|G(\eta_p)|^2}.
\end{equation}
Then, for every fixed $\theta\in\widehat\Theta$,
\begin{equation}
\label{eq:ecf_mse_mae_exact_section3}
\mathbb{E}\!\left[\mathrm{MSE}_{M,P}(\theta)\right]
=
\mathrm{MSE}_{P}^{*}(\theta)
+
\frac{V_P}{M},
\qquad
\mathbb{E}\!\left[R_{M,P}(\theta)\right]
\le
R_{P}^{*}(\theta)
+
\sqrt{\frac{2}{M}}\,W_P.
\end{equation}
Consequently,
\begin{equation}
\label{eq:ecf_loss_exact_section3}
\mathbb{E}\!\left[\mathrm{Loss}_{M,P}(\theta)\right]
\le
\mathrm{Loss}_{P}^{*}(\theta)
+
\frac{V_P}{M}
+
\lambda\sqrt{\frac{2}{M}}\,W_P.
\end{equation}
\end{lemma}
\begin{proof}[Proof of Lemma~\ref{lem:ecf_error_section3}]
For each fixed node $\eta_p$, write
\[
\widehat G_M(\eta_p)-\widehat G(\eta_p;\theta)
=
\bigl(G(\eta_p)-\widehat G(\eta_p;\theta)\bigr)
+
\bigl(\widehat G_M(\eta_p)-G(\eta_p)\bigr).
\]
Taking squared moduli and then expectations yields
\begin{equation}
\label{eq:start}
\mathbb{E}\left[
\big|\widehat G_M(\eta_p)-\widehat G(\eta_p;\theta)\big|^2
\right]
=
\big|G(\eta_p)-\widehat G(\eta_p;\theta)\big|^2
+
\mathbb{E}\left[
\big|\widehat G_M(\eta_p)-G(\eta_p)\big|^2
\right],
\end{equation}
since the cross term vanishes by $\mathbb{E}[\widehat G_M(\eta_p)-G(\eta_p)]=0$ from \eqref{eq:ecf_mean_var_section2}. Moreover,
\begin{equation}
\label{eq:var}
\mathbb{E}\left[
\big|\widehat G_M(\eta_p)-G(\eta_p)\big|^2
\right]
=
\mathrm{Var}\!\left(\widehat G_M(\eta_p)\right)
=
\frac{1-|G(\eta_p)|^2}{M},
\end{equation}
again by \eqref{eq:ecf_mean_var_section2}.
Now sum \eqref{eq:start} over $p=1,\dots,P$ and divide by $P$. Using \eqref{eq:var}, we obtain
\[
\frac{1}{P}\sum_{p=1}^P
\mathbb{E}\left[
\big|\widehat G_M(\eta_p)-\widehat G(\eta_p;\theta)\big|^2
\right]
=
\frac{1}{P}\sum_{p=1}^P
\big|G(\eta_p)-\widehat G(\eta_p;\theta)\big|^2
+
\frac{1}{MP}\sum_{p=1}^P (1-|G(\eta_p)|^2).
\]
Using the definitions of $\mathrm{MSE}_{M,P}(\theta)$, $\mathrm{MSE}_{P}^{*}(\theta)$, and $V_P$ yields the \mbox{first identity in \eqref{eq:ecf_mse_mae_exact_section3}.}

For the MAE term in \eqref{eq:ecf_mse_mae_exact_section3}, the triangle inequality gives
\begin{equation}
\label{eq:second_term}
\begin{aligned}
|\Delta_{p,\operatorname{Re}}(\theta)|
+|\Delta_{p,\operatorname{Im}}(\theta)|
\le\;&
|\Delta_{p,\operatorname{Re}}^{*}(\theta)|
+
|\Delta_{p,\operatorname{Im}}^{*}(\theta)|
+
|\operatorname{Re}(\widehat G_M(\eta_p)-G(\eta_p))|
\\
&+
|\operatorname{Im}(\widehat G_M(\eta_p)-G(\eta_p))|.
\end{aligned}
\end{equation}
The last two terms on the right-hand side of \eqref{eq:second_term} are bounded by\footnote{We use the inequality $|a|+|b|\le \sqrt{2}\,(a^2+b^2)^{1/2}$.}
\begin{equation}
\label{eq:third_equation}
|\operatorname{Re}(\widehat G_M(\eta_p)-G(\eta_p))|
+
|\operatorname{Im}(\widehat G_M(\eta_p)-G(\eta_p))|
\le
\sqrt{2}\,\big|\widehat G_M(\eta_p)-G(\eta_p)\big|.
\end{equation}
Taking expectations both sides of \eqref{eq:third_equation} and applying Cauchy--Schwarz yields
\[
\mathbb{E}\big[
|\operatorname{Re}(\widehat G_M(\eta_p)-G(\eta_p))|
+
|\operatorname{Im}(\widehat G_M(\eta_p)-G(\eta_p))|
\big]
\le
\sqrt{2}\,
\big(
\mathbb{E}\big|\widehat G_M(\eta_p)-G(\eta_p)\big|^2
\big)^{1/2}.
\]
This and \eqref{eq:var} give
\begin{equation}
\label{eq:bound2}
\mathbb{E}\big[
|\operatorname{Re}(\widehat G_M(\eta_p)-G(\eta_p))|
+
|\operatorname{Im}(\widehat G_M(\eta_p)-G(\eta_p))|
\big]
\le
\sqrt{\frac{2}{M}}\,
\sqrt{1-|G(\eta_p)|^2}.
\end{equation}
Now sum \eqref{eq:second_term} over $p=1,\dots,P$ and divide by $P$. Using \eqref{eq:bound2} and the definitions of $R_{M,P}(\theta)$, $R_P^*(\theta)$, and $W_P$ yields the second bound in \eqref{eq:ecf_mse_mae_exact_section3}.
Finally, \eqref{eq:ecf_loss_exact_section3} follows from \eqref{eq:ecf_mse_mae_exact_section3} together with \eqref{eq:Loss_MP_section2} and \eqref{eq:Loss_exact_section3}.
\end{proof}

\subsection{Discretization and truncation in the Fourier domain}
The next two lemmas quantify the discretization and truncation errors in the Fourier domain.
\begin{lemma}[Discretization]
\label{lem:discretization_section3}
Fix $\eta'>0$. For $\theta\in\Theta$, define
\[
F_\theta(\eta)=|G(\eta)-\widehat G(\eta;\theta)|^2,
\qquad \eta\in[-\eta',\eta'].
\]
Then there exists a finite constant $C'=C'(\eta')>0$, independent of $P$ and $\theta$, such that
\begin{equation}
\label{eq:disc_bound_section3}
\Big|
\sum_{p=1}^{P}\delta_p F_\theta(\eta_p)
-
\int_{-\eta'}^{\eta'}F_\theta(\eta)\,d\eta
\Big|
\le
\frac{C' C_1^2}{P},
\qquad \forall\,\theta\in\Theta.
\end{equation}
\end{lemma}
For a proof of Lemma~\ref{lem:discretization_section3}, see Appendix~\ref{app:discretization_section3}.
\begin{lemma}[Fourier-domain truncation error]
\label{lem:truncation_section3}
For every $\varepsilon>0$, the training interval $[-\eta',\eta']$ can be chosen sufficiently large so that
\begin{equation}
\label{eq:trunc_section3}
\int_{\Rbb\setminus[-\eta',\eta']} |G(\eta)|^2\,d\eta < \varepsilon,
\qquad
\sup_{\theta\in\Theta}
\int_{\Rbb\setminus[-\eta',\eta']} |\widehat G(\eta;\theta)|^2\,d\eta < \varepsilon.
\end{equation}
Consequently, for every $\theta\in\Theta$,
\begin{equation}
\label{eq:tail_bound_section3}
\int_{\Rbb\setminus[-\eta',\eta']} |G(\eta)-\widehat G(\eta;\theta)|^2\,d\eta
<
4\varepsilon.
\end{equation}
\end{lemma}
A proof of Lemma~\ref{lem:truncation_section3}
is given in Appendix~\ref{app:truncation_section3}.

\subsection{Main {\dpurple{squared}} $L_2$ {\dpurple{density-}}error bound}
We now state the main theorem giving an expected {\dpurple{squared}} $L_2$ {\dpurple{density-}}error bound for the direct i.i.d.\ sampling setting.
\begin{theorem}[Data-driven {\dpurple{expected squared}} $L_2$ error bound]
\label{thm:error_decomposition_section3}
Fix $\varepsilon>0$, and choose $[-\eta',\eta']$ sufficiently large so that Lemma~\ref{lem:truncation_section3} applies. Let $\widehat\theta\in\widehat\Theta$ be the parameter obtained by minimizing \eqref{eq:Loss_MP_section2}, and suppose that, for some {\dpurple{deterministic}} $\varepsilon_1>0$,
\begin{equation}
\label{eq:empirical_opt_section3}
\mathrm{Loss}_{M,P}(\widehat\theta)\le \varepsilon_1
\qquad {\dpurple{\text{almost surely}.}}
\end{equation}
Then
\begin{equation}
\label{eq:main_error_bound_section3}
{\dpurple{
\mathbb{E}\!\Big[
\int_{\Rbb}\big|g(x)-\widehat g(x;\widehat\theta)\big|^2\,dx
\Big]
<
\frac{1}{2\pi}
\Big(
4\varepsilon
+
C_1
\Big(
2\varepsilon_1
+
\frac{2V_P}{M}
\Big)
+
\frac{C' C_1^2}{P}
\Big).
}}
\end{equation}
{\dpurple{Here, $C_1$ is defined in \eqref{eq:eta_partition_section2}, $V_P$ in \eqref{eq:VWP_section3}, and $C'$ in Lemma~\ref{lem:discretization_section3}.}}
\end{theorem}

\begin{proof}[Proof of Theorem~\ref{thm:error_decomposition_section3}]
By Plancherel's theorem,
\begin{align}
\label{eq:Plancherel}
2\pi\int_{\Rbb}\big|g(x)-\widehat g(x;\widehat\theta)\big|^2\,dx
&=
\int_{\Rbb}\big|G(\eta)-\widehat G(\eta;\widehat\theta)\big|^2\,d\eta
\nonumber
\\
&=
\int_{\Rbb\setminus[-\eta',\eta']} \!\!\big|G(\eta)-\widehat G(\eta;\widehat\theta)\big|^2\,d\eta
+
\int_{-\eta'}^{\eta'}\big|G(\eta)-\widehat G(\eta;\widehat\theta)\big|^2\,d\eta.
\end{align}
The first term on the right-hand side is bounded by $4\varepsilon$ by Lemma~\ref{lem:truncation_section3}. For the second term, with
\[
F_{\widehat\theta}(\eta)=|G(\eta)-\widehat G(\eta;\widehat\theta)|^2,
\]
Lemma~\ref{lem:discretization_section3} gives
\[
\int_{-\eta'}^{\eta'}F_{\widehat\theta}(\eta)\,d\eta
\le
\sum_{p=1}^{P}\delta_p F_{\widehat\theta}(\eta_p)
+
\frac{C'C_1^2}{P}
\le
\frac{C_1}{P}\sum_{p=1}^{P}F_{\widehat\theta}(\eta_p)
+
\frac{C'C_1^2}{P},
\]
since $\delta_p\le C_1/P$. Therefore
\begin{equation}
\label{eq:boundF}
{\dpurple{
\int_{-\eta'}^{\eta'}F_{\widehat\theta}(\eta)\,d\eta
\le
C_1\,\mathrm{MSE}_{P}^{*}(\widehat\theta)
+
\frac{C'C_1^2}{P}.
}}
\end{equation}

{\dpurple{It remains to bound $\mathbb E[\mathrm{MSE}_{P}^{*}(\widehat\theta)]$.}} For each $p$,
\[
|G(\eta_p)-\widehat G(\eta_p;\widehat\theta)|^2
\le
2|\widehat G_M(\eta_p)-\widehat G(\eta_p;\widehat\theta)|^2
+
2|\widehat G_M(\eta_p)-G(\eta_p)|^2.
\]
{\dpurple{Averaging over $p=1,\ldots,P$ gives}}
\[
{\dpurple{
\mathrm{MSE}_{P}^{*}(\widehat\theta)
\le
2\,\mathrm{MSE}_{M,P}(\widehat\theta)
+
\frac{2}{P}\sum_{p=1}^{P}
|\widehat G_M(\eta_p)-G(\eta_p)|^2.
}}
\]
{\dpurple{Since $\lambda\ge0$, \eqref{eq:empirical_opt_section3} implies, almost surely,}}
\[
{\dpurple{
\mathrm{MSE}_{M,P}(\widehat\theta)
\le
\mathrm{Loss}_{M,P}(\widehat\theta)
\le
\varepsilon_1.
}}
\]
{\dpurple{Taking expectations and using \eqref{eq:var} yields}}
\begin{equation}
\label{eq:exact_mse_bound_section3}
{\dpurple{
\mathbb E\!\left[\mathrm{MSE}_{P}^{*}(\widehat\theta)\right]
\le
2\varepsilon_1
+
\frac{2V_P}{M}.
}}
\end{equation}
{\dpurple{Combining \eqref{eq:Plancherel}, \eqref{eq:boundF}, and \eqref{eq:exact_mse_bound_section3} yields \eqref{eq:main_error_bound_section3}.}}
\end{proof}

\begin{corollary}[{\dpurple{Expected $L_2$ norm}}]
\label{cor:expected_L2_norm_section3}
{\dpurple{Under the conditions of Theorem~\ref{thm:error_decomposition_section3},}}
\begin{equation}
\label{eq:expected_L2_norm_section3}
{\dpurple{
\mathbb E\!\left[
\left\|g-\widehat g(\cdot;\widehat\theta)\right\|_{L_2}
\right]
<
\left[
\frac{1}{2\pi}
\Big(
4\varepsilon
+
C_1
\Big(
2\varepsilon_1
+
\frac{2V_P}{M}
\Big)
+
\frac{C'C_1^2}{P}
\Big)
\right]^{1/2}.
}}
\end{equation}
\end{corollary}

\begin{proof}
{\dpurple{By Cauchy--Schwarz,}}
$
{\dpurple{
\mathbb E\!\big[
\big\|g-\widehat g(\cdot;\widehat\theta)\big\|_{L_2}
\big]
\le
\big(
\mathbb E\!\big[
\big\|g-\widehat g(\cdot;\widehat\theta)\big\|_{L_2}^2
\big]
\big)^{1/2}.
}}
$
{\dpurple{The result follows from Theorem~\ref{thm:error_decomposition_section3}.}}
\end{proof}

\subsection{{\dpurple{Discussion of the error bounds}}}
{\dpurple{Lemma~\ref{lem:ecf_error_section3} quantifies the effect of empirical-CF sampling on the complete MSE-plus-MAE training loss: the terms $V_P/M$ and $W_P/\sqrt{M}$ arise from its MSE and MAE components, respectively. Theorem~\ref{thm:error_decomposition_section3} uses the MSE component to bound the expected squared $L_2$ density error. Its four contributions are Fourier-domain truncation, represented by $\varepsilon$; empirical training error, represented by $\varepsilon_1$; discretization error, represented by $C'C_1^2/P$; and empirical-CF sampling error, represented by $V_P/M$.}}

{\dpurple{The quantities $V_P$ and $W_P$ measure, respectively, the average squared and absolute empirical-CF fluctuations across the training nodes. Both depend on $|G(\eta_p)|$: they are small near the origin and typically more pronounced at larger frequencies. Since $|G(\eta_p)|\le1$ for every $p$,}}
\[
{\dpurple{
0\le V_P\le1,
\qquad
0\le W_P\le1.
}}
\]
{\dpurple{Thus the empirical-CF contribution to the expected squared $L_2$ density error in Theorem~\ref{thm:error_decomposition_section3} is of order $M^{-1}$, uniformly in $P$. Corollary~\ref{cor:expected_L2_norm_section3} converts this result into a bound for the expected $L_2$ norm. In particular, an $M^{-1/2}$ contribution to the expected $L_2$ norm follows when the truncation, empirical training, and discretization terms are controlled at order $M^{-1}$ or smaller; the theorem does not impose such rates on these remaining terms by itself.}}

\section{Pseudo-samples from dependent data}
\label{sec:bootstrap}
Sections~\ref{sc:NN_Fourier}--\ref{sec:error} treat the idealized setting of direct i.i.d.\ sampling from the target law. In many applications, however, the available observations form a dependent time series, while the object of interest is the law or transition density over a fixed horizon $\tau>0$. A practical way to recover conditionally i.i.d.\ training data is to construct pseudo-samples from the observed history by a suitable resampling procedure; see, for example, \cite{lahiri2003resampling}. Representative techniques include the moving block bootstrap \cite{kunsch1989jackknife}, the stationary bootstrap \cite{politis1994stationary}, subsampling \cite{politis1999subsampling}, and, when an explicit parametric model is postulated, \mbox{the parametric bootstrap \cite{efron1993bootstrap,davison1997bootstrap}.}

A key advantage of the Fourier-mixture framework is that dependence in the observed data can be handled without changing the basic structure of the error analysis: conditional on the observed history, the argument of Section~\ref{sec:error} carries over, and the only new error source is a pseudo-law approximation term measuring the discrepancy between the resampling-based law and the true $\tau$-horizon law.

\subsection{Resampling-based pseudo-sampling}
\label{subsec:bootstrap_setup}

Fix a horizon $\tau>0$, interpreted as the length of one decision period or one transition step. Let $X_\tau$ denote a generic true increment over horizon $\tau$, and let $\mathcal L_\tau=\mathcal L(X_\tau)$ be its law. We assume that this target law on $\Rbb$ admits a density $g$ and CF $G$, exactly as in Assumption~\ref{ass:modelling}.

Let $\mathcal H_N=(Y_1,\ldots,Y_N)$ denote an observed dependent sample recorded at the sampling frequency available in the data. This sampling frequency need not coincide with the target horizon $\tau$. A resampling procedure applied to $\mathcal H_N$ is intended to approximate the target law $\mathcal L_\tau$ by a random pseudo-law $\mathcal L_{\tau,N}^{\dagger}$. For example, when the data are observed at a finer frequency than $\tau$, one may resample shorter-horizon observations by a dependence-preserving scheme and then aggregate them over intervals of length $\tau$.

We write
\[
\mathbb E^{\dagger}[\cdot]=\mathbb E[\cdot\,|\,\mathcal H_N],
\qquad
\operatorname{Var}^{\dagger}(\cdot)=\operatorname{Var}(\cdot\,|\,\mathcal H_N)
\]
for conditional expectation and variance given the observed history. {\dpurple{Throughout this section, conditional expectations, variances, and almost-sure statements are with respect to the pseudo-sample generation and, when applicable, any randomness in the training procedure, conditional on $\mathcal H_N$.}} The basic conditional assumption is that, given $\mathcal H_N$, the pseudo-samples
$X_1^{\dagger},\ldots,X_M^{\dagger}$ are i.i.d.\ from the common pseudo-law $\mathcal L_{\tau,N}^{\dagger}$, namely
\[
X_1^{\dagger},\ldots,X_M^{\dagger}\,\big|\,\mathcal H_N
\stackrel{\mathrm{i.i.d.}}{\sim}
\mathcal L_{\tau,N}^{\dagger},
\qquad
\mathcal L_{\tau,N}^{\dagger}\approx\mathcal L_\tau.
\]
Let $X^{\dagger}$ denote a generic draw from $\mathcal L_{\tau,N}^{\dagger}$. The corresponding pseudo-CF is
\begin{equation*}
G_N^{\dagger}(\eta)
=
\mathbb E^{\dagger}\!\left[e^{i\eta X^{\dagger}}\right],
\qquad \eta\in\Rbb.
\end{equation*}
The associated empirical pseudo-CF is
\begin{equation}
\label{eq:hatG_boot}
\widehat G_M^{\dagger}(\eta)
=
\frac{1}{M}\sum_{m=1}^{M} e^{i\eta X_m^{\dagger}}.
\end{equation}
On the training nodes $\{\eta_p\}_{p=1}^{P}\subset[-\eta',\eta']$, we define
\begin{equation}
\label{eq:bootstrap_bias_terms}
\begin{aligned}
B_{P,N}^{(1)}
&=
\frac{1}{P}\sum_{p=1}^{P}
\Big(
\big|\operatorname{Re}(G(\eta_p)-G_N^{\dagger}(\eta_p))\big|
+
\big|\operatorname{Im}(G(\eta_p)-G_N^{\dagger}(\eta_p))\big|
\Big),
\\
B_{P,N}^{(2)}
&=
\frac{1}{P}\sum_{p=1}^{P}\big|G(\eta_p)-G_N^{\dagger}(\eta_p)\big|^2.
\end{aligned}
\end{equation}
These quantities are random through their dependence on $\mathcal H_N$ and measure how accurately the pseudo-law reproduces the target law on the training Fourier grid.

The Fourier-domain pseudo-sample loss is defined exactly as in Section~\ref{subsec:training_loss_section2}, but with $\widehat G_M$ replaced by $\widehat G_M^{\dagger}$:
\begin{equation}
\label{eq:Loss_boot}
\mathrm{Loss}_{M,P}^{\dagger}(\theta)
=
\mathrm{MSE}_{M,P}^{\dagger}(\theta)
+
\lambda\,R_{M,P}^{\dagger}(\theta),
\qquad \theta\in\widehat\Theta\subseteq\Theta,
\end{equation}
where $\lambda\ge 0$ and 
\begin{equation}
\label{eq:MSE_boot}
\begin{aligned}
\mathrm{MSE}_{M,P}^{\dagger}(\theta)
&=
\frac{1}{P}\sum_{p=1}^{P}
\Big(
\Delta_{p,\operatorname{Re}}^{\dagger}(\theta)^2
+
\Delta_{p,\operatorname{Im}}^{\dagger}(\theta)^2
\Big),
\\
R_{M,P}^{\dagger}(\theta)
&=
\frac{1}{P}\sum_{p=1}^{P}
\Big(
|\Delta_{p,\operatorname{Re}}^{\dagger}(\theta)|
+
|\Delta_{p,\operatorname{Im}}^{\dagger}(\theta)|
\Big),
\end{aligned}
\end{equation}
with
\[
\Delta_{p,\operatorname{Re}}^{\dagger}(\theta)
=
\operatorname{Re}\widehat G_M^{\dagger}(\eta_p)
-
\operatorname{Re}\widehat G(\eta_p;\theta),
\qquad
\Delta_{p,\operatorname{Im}}^{\dagger}(\theta)
=
\operatorname{Im}\widehat G_M^{\dagger}(\eta_p)
-
\operatorname{Im}\widehat G(\eta_p;\theta).
\]
Similarly, define the exact loss associated with $\mathcal L_{\tau,N}^{\dagger}$ by replacing $\widehat G_M^{\dagger}$ with $G_N^{\dagger}$:
\begin{equation}
\label{eq:Loss_boot_exact}
\mathrm{Loss}_{P,N}^{\dagger,*}(\theta)
=
\mathrm{MSE}_{P,N}^{\dagger,*}(\theta)
+
\lambda\,R_{P,N}^{\dagger,*}(\theta),
\end{equation}
where $\mathrm{MSE}_{P,N}^{\dagger,*}(\theta)$ and $R_{P,N}^{\dagger,*}(\theta)$ are defined analogously to \eqref{eq:MSE_boot}, with residuals
\[
\Delta_{p,\operatorname{Re}}^{\dagger,*}(\theta)
=
\operatorname{Re}G_N^{\dagger}(\eta_p)
-
\operatorname{Re}\widehat G(\eta_p;\theta),
\qquad
\Delta_{p,\operatorname{Im}}^{\dagger,*}(\theta)
=
\operatorname{Im}G_N^{\dagger}(\eta_p)
-
\operatorname{Im}\widehat G(\eta_p;\theta).
\]

\subsection{Error analysis under resampling-based pseudo-sampling}
\label{subsec:bootstrap_error}
{\dpurple{We now separate the perturbation analysis for the complete MSE-plus-MAE pseudo-sample loss from the physical-space density-error analysis. The first result is the pseudo-sample analogue of Lemma~\ref{lem:ecf_error_section3}; the sharper density-error theorem that follows uses only the MSE component. }}
\begin{lemma}[Pseudo-sample unbiasedness and variance]
\label{lem:bootstrap_ecf}
For each fixed $\eta\in\Rbb$,
\begin{equation}
\label{eq:boot_mean}
\mathbb E^{\dagger}\!\left[\widehat G_M^{\dagger}(\eta)\right]
=
G_N^{\dagger}(\eta),
\qquad
\operatorname{Var}^{\dagger}\!\left(\widehat G_M^{\dagger}(\eta)\right)
=
\frac{1-|G_N^{\dagger}(\eta)|^2}{M}.
\end{equation}
Moreover, with
\begin{equation}
\label{eq:VW_boot}
V_{P,N}^{\dagger}
=
\frac{1}{P}\sum_{p=1}^{P}\big(1-|G_N^{\dagger}(\eta_p)|^2\big),
\qquad
W_{P,N}^{\dagger}
=
\frac{1}{P}\sum_{p=1}^{P}\sqrt{1-|G_N^{\dagger}(\eta_p)|^2},
\end{equation}
we have, for every fixed $\theta\in\widehat\Theta$,
\begin{equation}
\label{eq:boot_mse_mae}
\mathbb E^{\dagger}[\mathrm{MSE}_{M,P}^{\dagger}(\theta)]
=
\mathrm{MSE}_{P,N}^{\dagger,*}(\theta)
+
\frac{V_{P,N}^{\dagger}}{M},
\quad
\mathbb E^{\dagger}[R_{M,P}^{\dagger}(\theta)]
\le
R_{P,N}^{\dagger,*}(\theta)
+
\sqrt{\frac{2}{M}}\,W_{P,N}^{\dagger}.
\end{equation}
Consequently,
\begin{equation}
\label{eq:boot_loss_bound}
\mathbb E^{\dagger}\!\left[\mathrm{Loss}_{M,P}^{\dagger}(\theta)\right]
\le
\mathrm{Loss}_{P,N}^{\dagger,*}(\theta)
+
\frac{V_{P,N}^{\dagger}}{M}
+
\lambda\sqrt{\frac{2}{M}}\,W_{P,N}^{\dagger}.
\end{equation}
\end{lemma}
For a proof of Lemma~\ref{lem:bootstrap_ecf}, see Appendix~\ref{app:bootstrap_ecf}.

{\dpurple{Lemma~\ref{lem:bootstrap_ecf} quantifies the perturbation of the complete MSE-plus-MAE loss relative to the pseudo-law. To compare this full loss with its true-law counterpart, the definitions in \eqref{eq:bootstrap_bias_terms} give, for every $\theta\in\widehat\Theta$,}}
\begin{equation}
\label{eq:true_vs_boot_exact}
{\dpurple{
\mathrm{Loss}_{P}^{*}(\theta)
\le
2\,\mathrm{Loss}_{P,N}^{\dagger,*}(\theta)
+
2B_{P,N}^{(2)}
+
\lambda B_{P,N}^{(1)}.
}}
\end{equation}
{\dpurple{Thus $B_{P,N}^{(1)}$ and $W_{P,N}^{\dagger}/\sqrt{M}$ enter the perturbation analysis for the complete loss through its MAE component. The physical-space squared $L_2$ density error, however, is controlled by the squared Fourier residual alone, and therefore admits the sharper MSE-only bound below.}}

\begin{theorem}[{\dpurple{Conditional expected squared $L_2$ error bound}}]
\label{thm:bootstrap_error}
Fix $\varepsilon>0$, and choose $[-\eta',\eta']$ sufficiently large so that Lemma~\ref{lem:truncation_section3} applies. Let $\widehat\theta^{\dagger}\in\widehat\Theta$ be obtained by minimizing \eqref{eq:Loss_boot}, and suppose that, for some {\dpurple{deterministic}} $\varepsilon_1>0$,
\begin{equation}
\label{eq:boot_empirical_loss}
\mathrm{Loss}_{M,P}^{\dagger}(\widehat\theta^{\dagger})\le \varepsilon_1
\qquad {\dpurple{\text{almost surely}.}}
\end{equation}
Then
\begin{equation}
\label{eq:bootstrap_main_bound}
\begin{aligned}
{\dpurple{
\mathbb E^{\dagger}\Big[
\int_{\Rbb}\big|g(x)-\widehat g(x;\widehat\theta^{\dagger})\big|^2\,dx
\Big]
}}
&{\dpurple{
<
\frac{1}{2\pi}
\Big(
4\varepsilon
+
C_1
\Big(
4\varepsilon_1
+
2B_{P,N}^{(2)}
+
\frac{4V_{P,N}^{\dagger}}{M}
\Big)
+
\frac{C' C_1^2}{P}
\Big).
}}
\end{aligned}
\end{equation}
{\dpurple{Here, $C_1$ is defined in \eqref{eq:eta_partition_section2}, $B_{P,N}^{(2)}$ in \eqref{eq:bootstrap_bias_terms}, $V_{P,N}^{\dagger}$ in \eqref{eq:VW_boot}, and $C'$ in Lemma~\ref{lem:discretization_section3}.}}
\end{theorem}

\begin{proof}
By Plancherel's theorem,
\[
2\pi\int_{\Rbb}\big|g(x)-\widehat g(x;\widehat\theta^{\dagger})\big|^2\,dx
=
\int_{\Rbb}\big|G(\eta)-\widehat G(\eta;\widehat\theta^{\dagger})\big|^2\,d\eta.
\]
Applying the same truncation and discretization arguments as in the proof of Theorem~\ref{thm:error_decomposition_section3} gives
\begin{equation}
\label{eq:boot_start}
{\dpurple{
\int_{\Rbb}\big|G(\eta)-\widehat G(\eta;\widehat\theta^{\dagger})\big|^2\,d\eta
<
4\varepsilon
+
C_1\,\mathrm{MSE}_{P}^{*}(\widehat\theta^{\dagger})
+
\frac{C' C_1^2}{P}.
}}
\end{equation}

{\dpurple{We next compare the true-law MSE with its pseudo-law counterpart. Since}}
\[
|G(\eta_p)-\widehat G(\eta_p;\widehat\theta^{\dagger})|^2
\le
2|G(\eta_p)-G_N^{\dagger}(\eta_p)|^2
+
2|G_N^{\dagger}(\eta_p)-\widehat G(\eta_p;\widehat\theta^{\dagger})|^2,
\]
{\dpurple{averaging over $p=1,\ldots,P$ gives}}
\begin{equation}
\label{eq:true_vs_boot_mse}
{\dpurple{
\mathrm{MSE}_{P}^{*}(\widehat\theta^{\dagger})
\le
2B_{P,N}^{(2)}
+
2\,\mathrm{MSE}_{P,N}^{\dagger,*}(\widehat\theta^{\dagger}).
}}
\end{equation}
{\dpurple{For each $p$,}}
\[
{\dpurple{
|G_N^{\dagger}(\eta_p)-\widehat G(\eta_p;\widehat\theta^{\dagger})|^2
\le
2|\widehat G_M^{\dagger}(\eta_p)-\widehat G(\eta_p;\widehat\theta^{\dagger})|^2
+
2|\widehat G_M^{\dagger}(\eta_p)-G_N^{\dagger}(\eta_p)|^2.
}}
\]
{\dpurple{Averaging over $p=1,\ldots,P$, using $\lambda\ge0$ and the almost-sure bound \eqref{eq:boot_empirical_loss}, and then taking $\mathbb E^{\dagger}$ yields}}
\begin{equation}
\label{eq:boot_exact_mse_bound}
{\dpurple{
\mathbb E^{\dagger}\!\left[
\mathrm{MSE}_{P,N}^{\dagger,*}(\widehat\theta^{\dagger})
\right]
\le
2\varepsilon_1
+
\frac{2V_{P,N}^{\dagger}}{M}.
}}
\end{equation}
{\dpurple{Combining \eqref{eq:true_vs_boot_mse} and \eqref{eq:boot_exact_mse_bound} gives}}
\[
{\dpurple{
\mathbb E^{\dagger}\!\left[
\mathrm{MSE}_{P}^{*}(\widehat\theta^{\dagger})
\right]
\le
4\varepsilon_1
+
2B_{P,N}^{(2)}
+
\frac{4V_{P,N}^{\dagger}}{M}.
}}
\]
{\dpurple{Substituting this estimate into \eqref{eq:boot_start} proves \eqref{eq:bootstrap_main_bound}.}}
\end{proof}

\begin{corollary}[{\dpurple{Conditional expected $L_2$ norm}}]
\label{cor:expected_L2_norm_bootstrap}
{\dpurple{Under the conditions of Theorem~\ref{thm:bootstrap_error},}}
\begin{equation}
\label{eq:expected_L2_norm_bootstrap}
\begin{aligned}
{\dpurple{
\mathbb E^{\dagger}\!\left[
\left\|g-\widehat g(\cdot;\widehat\theta^{\dagger})\right\|_{L_2}
\right]
}}
&{\dpurple{
<
\left[
\frac{1}{2\pi}
\Big(
4\varepsilon
+
C_1
\Big(
4\varepsilon_1
+
2B_{P,N}^{(2)}
+
\frac{4V_{P,N}^{\dagger}}{M}
\Big)
+
\frac{C' C_1^2}{P}
\Big)
\right]^{1/2}.
}}
\end{aligned}
\end{equation}
\end{corollary}

\begin{proof}
{\dpurple{By conditional Cauchy--Schwarz,}}
$
{\dpurple{
\mathbb E^{\dagger}\!\big[
\big\|g-\widehat g(\cdot;\widehat\theta^{\dagger})\big\|_{L_2}
\big]
\le
\big(
\mathbb E^{\dagger}\!\big[
\big\|g-\widehat g(\cdot;\widehat\theta^{\dagger})\big\|_{L_2}^2
\big]
\big)^{1/2}.
}}
$
{\dpurple{The result follows from Theorem~\ref{thm:bootstrap_error}.}}
\end{proof}

\begin{remark}[Interpretation of the pseudo-law terms]
\label{rem:bootstrap_terms}
{\dpurple{Lemma~\ref{lem:bootstrap_ecf}, together with \eqref{eq:true_vs_boot_exact}, describes the perturbation of the complete MSE-plus-MAE loss. The quantities $V_{P,N}^{\dagger}/M$ and $W_{P,N}^{\dagger}/\sqrt M$ quantify, respectively, the conditional squared and absolute fluctuations of the empirical pseudo-CF around the pseudo-CF, while $B_{P,N}^{(2)}$ and $B_{P,N}^{(1)}$ quantify the corresponding squared and absolute discrepancies between the pseudo-law and the true law on the Fourier grid. Theorem~\ref{thm:bootstrap_error}, by contrast, controls the conditional expected squared $L_2$ density error through the MSE component alone. Its additional pseudo-sampling terms are therefore $B_{P,N}^{(2)}$ and $V_{P,N}^{\dagger}/M$; the quantities $B_{P,N}^{(1)}$ and $W_{P,N}^{\dagger}/\sqrt M$ remain relevant to the MAE component of the full training loss but are not needed in the sharper density-error bound.}}
\end{remark}

\begin{remark}[Resampling procedures and practical interpretation]
\label{rem:bootstrap_consistency}
Assume that the resampling procedure is consistent for the fixed-horizon law on the chosen Fourier grid, in the sense that
\[
B_{P,N}^{(2)}\to 0,
\qquad
B_{P,N}^{(1)}\to 0,
\qquad N\to\infty,
\]
in probability or in mean. {\dpurple{Then the pseudo-law discrepancy terms vanish asymptotically in both the full-loss comparison and the density-error bound.}} Moreover, since $G_N^{\dagger}$ is a CF,
$|G_N^{\dagger}(\eta_p)|\le 1$ for every $p$, and thus we have
\[
0\le V_{P,N}^{\dagger}\le 1,
\qquad
0\le W_{P,N}^{\dagger}\le 1.
\]
{\dpurple{Hence Theorem~\ref{thm:bootstrap_error} implies the simplified squared $L_2$ bound}}
\[
{\dpurple{
\mathbb E^{\dagger}\!\Big[
\int_{\Rbb}\big|g(x)-\widehat g(x;\widehat\theta^{\dagger})\big|^2\,dx
\Big]
<
\frac{1}{2\pi}
\Big(
4\varepsilon
+
C_1
\Big(
4\varepsilon_1
+
2B_{P,N}^{(2)}
+
\frac{4}{M}
\Big)
+
\frac{C' C_1^2}{P}
\Big).
}}
\]
{\dpurple{Thus the $M$- and $P$-dependent part of the conditional squared $L_2$ density-error bound has the same structure as in the direct i.i.d.\ setting, together with the additional $N$-dependent term $B_{P,N}^{(2)}$. Corollary~\ref{cor:expected_L2_norm_bootstrap} gives the corresponding conditional expected $L_2$-norm bound.}} This provides a theoretical basis for real-data experiments in which a dependent time series is first converted into approximately i.i.d.\ pseudo-samples by a resampling procedure and the Fourier-mixture model is then trained on the corresponding empirical pseudo-CF. For financial time series, block bootstrap methods are a natural example; see Section~\ref{sec:num}.
\end{remark}

\section{Multi-dimensional extension}
\label{sc:multi}
We now extend the construction to dimension $d\ge 1$.
{\dpurple{The model construction below is defined for arbitrary $d$. For the present $L_2$ analysis, however, the Gaussian--Laplace class with $K_{\mathcal L}>0$ is restricted to $d\in\{1,2,3\}$; the Gaussian-only subclass with $K_{\mathcal L}=0$ remains admissible for arbitrary $d\ge1$.}}
We distinguish between the direct i.i.d.\ sampling setting and its pseudo-sample extension.

\subsection{Direct i.i.d.\ sampling}
Throughout this section, bold symbols denote vectors or matrices in $\Rbb^d$ or $\Rbb^{d\times d}$. Let $\mathbf{X}_1,\ldots,\mathbf{X}_M$ be i.i.d.\ observations in $\Rbb^d$ from an unknown density $g$, and let $G$ denote the Fourier transform of $g$. We seek to estimate $g$ from the empirical CF $\widehat G_M$. {\dpurple{All expectations and almost-sure statements in this subsection are with respect to the i.i.d.\ sample and, when applicable, any randomness in the training procedure.}}
Our Fourier convention is
\begin{equation}
\label{eq:FT_pair_multi}
G(\boldsymbol{\eta})
=
\int_{\Rbb^d} e^{i\boldsymbol{\eta}^{\top}\mathbf{x}}\,g(\mathbf{x})\,d\mathbf{x},
\qquad
g(\mathbf{x})
=
\frac{1}{(2\pi)^d}\int_{\Rbb^d} e^{-i\boldsymbol{\eta}^{\top}\mathbf{x}}\,G(\boldsymbol{\eta})\,d\boldsymbol{\eta},
\end{equation}
and the empirical CF is
\begin{equation}
\label{eq:ecf_multi}
\widehat G_M(\boldsymbol{\eta})
=
\frac{1}{M}\sum_{m=1}^{M} e^{i\boldsymbol{\eta}^{\top}\mathbf{X}_m},
\qquad
\boldsymbol{\eta}\in\Rbb^d.
\end{equation}

For $\boldsymbol{\mu},\boldsymbol{\nu}\in\Rbb^d$ and symmetric positive-definite matrices $\mathbf{\Sigma},\mathbf{\Lambda}\in\Rbb^{d\times d}$, define the multivariate Gaussian and elliptical Laplace kernels
\begin{equation}
\label{eq:multi_gauss_kernel}
\varphi_{\mathcal G}^{(d)}(\mathbf{x};\boldsymbol{\mu},\mathbf{\Sigma})
=
\frac{1}{(2\pi)^{d/2}|\mathbf{\Sigma}|^{1/2}}
\exp\!\left(
-\frac{1}{2}(\mathbf{x}-\boldsymbol{\mu})^{\top}\mathbf{\Sigma}^{-1}(\mathbf{x}-\boldsymbol{\mu})
\right),
\end{equation}
\begin{equation}
\label{eq:multi_laplace_kernel}
\varphi_{\mathcal L}^{(d)}(\mathbf{x};\boldsymbol{\nu},\mathbf{\Lambda})
=
\frac{1}{(2\pi)^{d/2}|\mathbf{\Lambda}|^{1/2}}\,
r(\mathbf{x};\boldsymbol{\nu},\mathbf{\Lambda})^{1-d/2}\,
K_{d/2-1}\!\big(r(\mathbf{x};\boldsymbol{\nu},\mathbf{\Lambda})\big),
\end{equation}
where
\[
r(\mathbf{x};\boldsymbol{\nu},\mathbf{\Lambda})
=
\sqrt{(\mathbf{x}-\boldsymbol{\nu})^{\top}\mathbf{\Lambda}^{-1}(\mathbf{x}-\boldsymbol{\nu})},
\]
and $K_{\alpha}(\cdot)$ is the modified Bessel function of the second kind.
{\blue{Here, $\mathbf{\Sigma}$ is the Gaussian covariance matrix, while $\mathbf{\Lambda}$ is the Laplace scale matrix; the covariance matrix of the corresponding Laplace component is $2\mathbf{\Lambda}$.}}

The corresponding Gaussian--Laplace mixture is
\begin{equation}
\label{eq:Sigma_GL_multi}
\widehat g(\mathbf{x};\theta)
=
\sum_{k=1}^{K_{\mathcal G}}
\beta_k^{(\mathcal G)}
\varphi_{\mathcal G}^{(d)}(\mathbf{x};\boldsymbol{\mu}_k,\mathbf{\Sigma}_k)
+
\sum_{\ell=1}^{K_{\mathcal L}}
\beta_\ell^{(\mathcal L)}
\varphi_{\mathcal L}^{(d)}(\mathbf{x};\boldsymbol{\nu}_\ell,\mathbf{\Lambda}_\ell),
\end{equation}
with nonnegative weights summing to one. As in 1D, the weights are obtained by a joint softmax, while {\blue{the Gaussian covariance matrices and Laplace scale matrices}} are parameterized by Cholesky factors with softplus-transformed diagonals.

{\blue{For the analysis, we impose the multidimensional analogue of \eqref{eq:Theta_GL_section2}. Specifically, the component locations and matrices satisfy
\[
\|\boldsymbol{\mu}_k\|,\|\boldsymbol{\nu}_\ell\|\le \overline\mu,
\qquad
\sigma_{\min}^2\mathbf I_d\preceq\mathbf{\Sigma}_k\preceq\sigma_{\max}^2\mathbf I_d,
\qquad
b_{\min}^2\mathbf I_d\preceq\mathbf{\Lambda}_\ell\preceq b_{\max}^2\mathbf I_d,
\]
for all Gaussian and Laplace components. Throughout this section, $\Theta$ denotes the resulting multidimensional effective-parameter class, together with the simplex constraints on the mixture weights.}}

The learned CF again admits a closed form:
\begin{equation}
\label{eq:multi_cf}
\widehat G(\boldsymbol{\eta};\theta)
=
\sum_{k=1}^{K_{\mathcal G}}
\beta_k^{(\mathcal G)}
\exp\!\left(
i\boldsymbol{\eta}^{\top}\boldsymbol{\mu}_k
-
\frac{1}{2}\boldsymbol{\eta}^{\top}\mathbf{\Sigma}_k\boldsymbol{\eta}
\right)
+
\sum_{\ell=1}^{K_{\mathcal L}}
\beta_\ell^{(\mathcal L)}
\frac{\exp\!\left(i\boldsymbol{\eta}^{\top}\boldsymbol{\nu}_\ell\right)}
{1+\boldsymbol{\eta}^{\top}\mathbf{\Lambda}_\ell\boldsymbol{\eta}}.
\end{equation}
Accordingly,
\begin{align}
\label{eq:multi_cf_real}
\operatorname{Re}\widehat G(\boldsymbol{\eta};\theta)
&=
\sum_{k=1}^{K_{\mathcal G}}
\beta_k^{(\mathcal G)}
\exp\!\left(
-\frac{1}{2}\boldsymbol{\eta}^{\top}\mathbf{\Sigma}_k\boldsymbol{\eta}
\right)
\cos\!\left(\boldsymbol{\eta}^{\top}\boldsymbol{\mu}_k\right)
+
\sum_{\ell=1}^{K_{\mathcal L}}
\beta_\ell^{(\mathcal L)}
\frac{\cos\!\left(\boldsymbol{\eta}^{\top}\boldsymbol{\nu}_\ell\right)}
{1+\boldsymbol{\eta}^{\top}\mathbf{\Lambda}_\ell\boldsymbol{\eta}},
\\
\label{eq:multi_cf_imag}
\operatorname{Im}\widehat G(\boldsymbol{\eta};\theta)
&=
\sum_{k=1}^{K_{\mathcal G}}
\beta_k^{(\mathcal G)}
\exp\!\left(
-\frac{1}{2}\boldsymbol{\eta}^{\top}\mathbf{\Sigma}_k\boldsymbol{\eta}
\right)
\sin\!\left(\boldsymbol{\eta}^{\top}\boldsymbol{\mu}_k\right)
+
\sum_{\ell=1}^{K_{\mathcal L}}
\beta_\ell^{(\mathcal L)}
\frac{\sin\!\left(\boldsymbol{\eta}^{\top}\boldsymbol{\nu}_\ell\right)}
{1+\boldsymbol{\eta}^{\top}\mathbf{\Lambda}_\ell\boldsymbol{\eta}}.
\end{align}
The restriction of the Gaussian--Laplace $L_2$ analysis to $d\in\{1,2,3\}$ follows from the square-integrability of the elliptical Laplace kernel; see Appendix~\ref{app:multi_laplace_L2}.

Training proceeds exactly as in Subsection~\ref{subsec:training_loss_section2}, with Fourier domain $[-\eta',\eta']^d$. For training nodes $\{\boldsymbol{\eta}_p\}_{p=1}^{P}\subset[-\eta',\eta']^d$, define the residuals
\[
\Delta_{p,\operatorname{Re}}(\theta)
=
\operatorname{Re}\widehat G_M(\boldsymbol{\eta}_p)
-
\operatorname{Re}\widehat G(\boldsymbol{\eta}_p;\theta),
\qquad
\Delta_{p,\operatorname{Im}}(\theta)
=
\operatorname{Im}\widehat G_M(\boldsymbol{\eta}_p)
-
\operatorname{Im}\widehat G(\boldsymbol{\eta}_p;\theta),
\]
and use the loss
\begin{equation}
\label{eq:multi_loss}
\mathrm{Loss}_{M,P}^{(d)}(\theta)
=
\frac{1}{P}\sum_{p=1}^{P}
\Big(
\Delta_{p,\operatorname{Re}}(\theta)^2
+
\Delta_{p,\operatorname{Im}}(\theta)^2
\Big)
+
\frac{\lambda}{P}\sum_{p=1}^{P}
\Big(
|\Delta_{p,\operatorname{Re}}(\theta)|
+
|\Delta_{p,\operatorname{Im}}(\theta)|
\Big).
\end{equation}
Assume that $P$ is a perfect $d$th power and that $[-\eta',\eta']^d$ is partitioned by a deterministic tensor-product grid with $P^{1/d}$ nodes in each coordinate direction. Let $\delta_{\max}=C_1/P^{1/d}$. We define
\begin{equation}
\label{eq:multi_VW}
V_P^{(d)}
=
\frac{1}{P}\sum_{p=1}^{P}\big(1-|G(\boldsymbol{\eta}_p)|^2\big),
\qquad
W_P^{(d)}
=
\frac{1}{P}\sum_{p=1}^{P}\sqrt{1-|G(\boldsymbol{\eta}_p)|^2}.
\end{equation}
{\dpurple{As in Lemma~\ref{lem:ecf_error_section3}, $V_P^{(d)}/M$ and $W_P^{(d)}/\sqrt{M}$ quantify the perturbations of the MSE and MAE components, respectively, of the complete training loss. The physical-space squared $L_2$ density error is controlled by the squared Fourier residual alone, and therefore only $V_P^{(d)}/M$ appears in the sharper bound below.}}

\begin{theorem}[{\dpurple{Expected squared $L_2$ error bound in dimension $d$}}]
\label{thm:multi_error}
{\dpurple{Assume either that $1\le d\le3$, or that $d\ge1$ and $K_{\mathcal L}=0$.}}
Assume that $g$ is a probability density in $L_\infty(\Rbb^d)$ and that $G$ is Lipschitz continuous on every compact subset of $\Rbb^d$. Fix $\varepsilon>0$, and choose $\eta'>0$ sufficiently large so that the Fourier tails of $G$ and $\widehat G(\cdot;\theta)$ outside $[-\eta',\eta']^d$ are bounded by $\varepsilon$, uniformly in $\theta\in\Theta$. Let $\widehat\theta$ be obtained by minimizing \eqref{eq:multi_loss}, and suppose that, for some {\dpurple{deterministic}} $\varepsilon_1>0$,
\[
\mathrm{Loss}_{M,P}^{(d)}(\widehat\theta)\le \varepsilon_1
\qquad {\dpurple{\text{almost surely}.}}
\]
Then
\begin{equation*}
{\dpurple{
\mathbb{E}\Big[
\int_{\Rbb^d}\big|g(\mathbf{x})-\widehat g(\mathbf{x};\widehat\theta)\big|^2\,d\mathbf{x}
\Big]
<
\frac{1}{(2\pi)^d}
\Big(
4\varepsilon
+
C_1^d
\big(
2\varepsilon_1
+
\frac{2V_P^{(d)}}{M}
\big)
+
\frac{C_d' C_1^{d+1}}{P^{1/d}}
\Big),
}}
\end{equation*}
where $C_d'=C_d'(\eta')$ is independent of $P$. Since $|G(\boldsymbol{\eta}_p)|\le 1$ for every $p$, {\dpurple{we have $0\le V_P^{(d)}\le1$, and hence}}
\begin{equation*}
{\dpurple{
\mathbb{E}\Big[
\int_{\Rbb^d}\big|g(\mathbf{x})-\widehat g(\mathbf{x};\widehat\theta)\big|^2\,d\mathbf{x}
\Big]
<
\frac{1}{(2\pi)^d}
\Big(
4\varepsilon
+
C_1^d
\big(
2\varepsilon_1
+
\frac{2}{M}
\big)
+
\frac{C_d' C_1^{d+1}}{P^{1/d}}
\Big).
}}
\end{equation*}
\end{theorem}

\begin{corollary}[{\dpurple{Expected $L_2$ norm in dimension $d$}}]
\label{cor:multi_expected_L2_norm}
{\dpurple{Under the conditions of Theorem~\ref{thm:multi_error},}}
\begin{equation*}
{\dpurple{
\mathbb E\!\left[
\left\|g-\widehat g(\cdot;\widehat\theta)\right\|_{L_2(\Rbb^d)}
\right]
<
\left[
\frac{1}{(2\pi)^d}
\Big(
4\varepsilon
+
C_1^d
\big(
2\varepsilon_1
+
\frac{2V_P^{(d)}}{M}
\big)
+
\frac{C_d' C_1^{d+1}}{P^{1/d}}
\Big)
\right]^{1/2}.
}}
\end{equation*}
\end{corollary}

{\dpurple{For brevity, proofs of Theorem~\ref{thm:multi_error} and Corollary~\ref{cor:multi_expected_L2_norm} are omitted, since the arguments follow the same steps as in Section~\ref{sec:error}: Plancherel's theorem is replaced by its $d$-dimensional counterpart, and the 1D quadrature estimate is replaced by the tensor-product version on $[-\eta',\eta']^d$. The perturbation of the complete MSE-plus-MAE loss retains the terms $V_P^{(d)}/M$ and $W_P^{(d)}/\sqrt M$, while the sharper expected squared $L_2$ density-error bound uses the MSE component alone and contains only $V_P^{(d)}/M$. The deterministic discretization term deteriorates to order $P^{-1/d}$.}}

\subsection{Resampling-based pseudo-sampling}
{\dpurple{The pseudo-sample framework of Section~\ref{sec:bootstrap} also extends directly to the Gaussian--Laplace class for $1\le d\le3$, and to the Gaussian-only subclass for arbitrary $d\ge1$.}}
Let $\mathcal H_N$ denote the observed dependent history, and suppose that, conditional on $\mathcal H_N$, the pseudo-samples
\[
\mathbf{X}_1^{\dagger},\ldots,\mathbf{X}_M^{\dagger}\in\Rbb^d
\qquad\text{are i.i.d.\ from a common pseudo-law}\qquad
\mathcal L_{\tau,N}^{\dagger,(d)}.
\]
{\dpurple{Throughout this subsection, conditional expectations and almost-sure statements are with respect to the pseudo-sample generation and, when applicable, any randomness in the training procedure, conditional on $\mathcal H_N$.}}
Let $G_N^{\dagger,(d)}$ denote the CF of this pseudo-law, and define the empirical pseudo-CF by
\begin{equation}
\label{eq:ecf_multi_boot}
\widehat G_M^{\dagger,(d)}(\boldsymbol{\eta})
=
\frac{1}{M}\sum_{m=1}^{M} e^{i\boldsymbol{\eta}^{\top}\mathbf{X}_m^{\dagger}}.
\end{equation}
Define the multivariate pseudo-law discrepancy terms
\begin{equation}
\label{eq:multi_boot_bias}
\begin{aligned}
B_{P,N}^{(1,d)}
&=
\frac{1}{P}\sum_{p=1}^{P}
\Big(
\big|\operatorname{Re}(G(\boldsymbol{\eta}_p)-G_N^{\dagger,(d)}(\boldsymbol{\eta}_p))\big|
+
\big|\operatorname{Im}(G(\boldsymbol{\eta}_p)-G_N^{\dagger,(d)}(\boldsymbol{\eta}_p))\big|
\Big),
\\
B_{P,N}^{(2,d)}
&=
\frac{1}{P}\sum_{p=1}^{P}
\big|G(\boldsymbol{\eta}_p)-G_N^{\dagger,(d)}(\boldsymbol{\eta}_p)\big|^2,
\end{aligned}
\end{equation}
as well as
\begin{equation}
\label{eq:multi_boot_VW}
V_{P,N}^{\dagger,(d)}
=
\frac{1}{P}\sum_{p=1}^{P}\big(1-|G_N^{\dagger,(d)}(\boldsymbol{\eta}_p)|^2\big),
\qquad
W_{P,N}^{\dagger,(d)}
=
\frac{1}{P}\sum_{p=1}^{P}\sqrt{1-|G_N^{\dagger,(d)}(\boldsymbol{\eta}_p)|^2}.
\end{equation}
{\dpurple{As in Section~\ref{sec:bootstrap}, $B_{P,N}^{(1,d)}$ and $W_{P,N}^{\dagger,(d)}/\sqrt M$ enter the perturbation analysis for the MAE component of the complete loss, whereas the conditional expected squared $L_2$ density-error bound uses only $B_{P,N}^{(2,d)}$ and $V_{P,N}^{\dagger,(d)}/M$.}}
If $\widehat\theta^{\dagger}$ is obtained by minimizing the multivariate pseudo-sample loss and, for some {\dpurple{deterministic}} $\varepsilon_1>0$, satisfies
\[
\mathrm{Loss}_{M,P}^{\dagger,(d)}(\widehat\theta^{\dagger})\le \varepsilon_1
\qquad {\dpurple{\text{almost surely},}}
\]
then the same argument as in Section~\ref{sec:bootstrap} yields the conditional estimate
\begin{equation}
\label{eq:multi_bootstrap_bound}
\begin{aligned}
{\dpurple{
\mathbb E^{\dagger}\!\Big[
\int_{\Rbb^d}\big|g(\mathbf{x})-\widehat g(\mathbf{x};\widehat\theta^{\dagger})\big|^2\,d\mathbf{x}
\Big]
<
\frac{1}{(2\pi)^d}
\Big(
4\varepsilon
+
C_1^d
\big(
4\varepsilon_1
+
2B_{P,N}^{(2,d)}
+
\frac{4V_{P,N}^{\dagger,(d)}}{M}
\big)
+
\frac{C_d' C_1^{d+1}}{P^{1/d}}
\Big).
}}
\end{aligned}
\end{equation}
{\dpurple{Consequently, by conditional Cauchy--Schwarz,}}
\begin{equation*}
{\dpurple{
\mathbb E^{\dagger}\!\left[
\left\|g-\widehat g(\cdot;\widehat\theta^{\dagger})\right\|_{L_2(\Rbb^d)}
\right]
<
\left[
\frac{1}{(2\pi)^d}
\Big(
4\varepsilon
+
C_1^d
\big(
4\varepsilon_1
+
2B_{P,N}^{(2,d)}
+
\frac{4V_{P,N}^{\dagger,(d)}}{M}
\big)
+
\frac{C_d' C_1^{d+1}}{P^{1/d}}
\Big)
\right]^{1/2}.
}}
\end{equation*}
{\dpurple{Thus the complete pseudo-sample loss retains the squared and absolute pseudo-CF fluctuation terms and both pseudo-law discrepancy terms, while the conditional expected squared $L_2$ density error retains truncation, empirical training error, discretization, the empirical pseudo-CF term $V_{P,N}^{\dagger,(d)}/M$, and the squared pseudo-law discrepancy $B_{P,N}^{(2,d)}$.}}

\section{Complexity analysis}
\label{sc:complex}
We now discuss the computational complexity in both the direct i.i.d.\ sampling setting and the resampling-based pseudo-sampling setting. Recall that $M$ denotes the number of samples used to construct the empirical CF, $P$ the number of training frequencies, and $(K_{\mathcal G},K_{\mathcal L})$ the numbers of Gaussian and Laplace components. We write $K=K_{\mathcal G}+K_{\mathcal L}$
and let $B$ and $E$ respectively denote the minibatch size and the total number of training epochs in Algorithm~\ref{alg:data_driven_fournet}.
%
\subsection{Direct i.i.d.\ sampling}
\paragraph{One-dimensional case.}
The computational cost of Algorithm~\ref{alg:data_driven_fournet} is determined by:
(i) preprocessing of the empirical CF,
(ii) repeated forward evaluations of the learned CF on minibatches of training nodes,
and (iii) the corresponding backward passes and parameter updates.

The empirical CF values are computed via
$\widehat G_M(\eta_p)=\frac{1}{M}\sum_{m=1}^{M} e^{i\eta_p X_m}$, $p=1,\ldots,P$,
which costs $\mathcal O(MP)$. This is a one-off cost for a fixed dataset and node set. Optional affine preprocessing requires only $\mathcal O(M)$ operations.

At one training node $\eta_p$, each Gaussian component contributes
$\beta_k^{(\mathcal G)}e^{i\eta_p\mu_k-\sigma_k^2\eta_p^2/2}$,
and each Laplace component contributes
$\beta_\ell^{(\mathcal L)}
\frac{e^{i\eta_p\nu_\ell}}{1+b_\ell^2\eta_p^2}$.
Hence, in the forward pass,  the  cost at one node is $\mathcal O(K)$, and over a minibatch of size $B$ it is $\mathcal O(BK)$. The softmax and softplus evaluations contribute only $\mathcal O(K)$ per forward pass.
The loss function evaluation is $\mathcal O(B)$, and the backward pass plus parameter update has the same asymptotic order as the forward pass. Therefore one optimization step costs $\mathcal O(BK)$, one epoch costs $\mathcal O(PK)$, and $E$ epochs cost $\mathcal O(EPK)$.

Combining preprocessing and training gives
\begin{equation}
\label{eq:complexity_total_1d}
\mathcal O(MP)+\mathcal O(EPK).
\end{equation}
Thus, once $\widehat G_M$ is formed, the iterative part of the algorithm no longer \mbox{depends on $M$.}

\paragraph{$d$-dimensional case.}
If the model is extended to $\Rbb^d$ with full covariance matrices, then constructing the empirical CF at $P$ nodes costs $\mathcal O(MPd)$, since each factor $e^{i\boldsymbol{\eta}_p^{\top}\mathbf X_m}$ requires one inner product in $\Rbb^d$. Each Gaussian or Laplace component at one training node requires evaluating a linear form $\boldsymbol{\eta}^{\top}\boldsymbol{\mu}$ and a quadratic form $\boldsymbol{\eta}^{\top}\mathbf{\Sigma}\boldsymbol{\eta}$, which cost $\mathcal O(d)$ and $\mathcal O(d^2)$, respectively. Hence,  the forward cost over a minibatch of size $B$ is $\mathcal O(BKd^2)$, and the total training cost over $E$ epochs is $\mathcal O(EPKd^2)$. Therefore, the total complexity in the full-covariance setting is
\[
\mathcal O(MPd)+\mathcal O(EPKd^2).
\]
This assumes that the covariance matrices are handled through their Cholesky factors, so that $\boldsymbol{\eta}^{\top}\mathbf{\Sigma}\boldsymbol{\eta}$ is evaluated directly without explicitly forming $\mathbf{\Sigma}=\mathbf L\mathbf L^{\top}$.

\subsection{Resampling-based pseudo-sampling}

In the pseudo-sample framework of Section~\ref{sec:bootstrap}, the empirical CF is built from pseudo-samples rather than directly from the observed dependent sample. The total complexity therefore acquires an additional preprocessing term corresponding to pseudo-sample generation.

\paragraph{One-dimensional case.}
Let $\mathcal C_{\mathrm{res}}(N,M,\tau)$ denote the cost of generating $M$ pseudo-samples over horizon $\tau$ from the observed history $\mathcal H_N$. This cost depends on the chosen resampling procedure; for block-resampling methods for dependent data, see, for example, \cite{kunsch1989jackknife,politis1994stationary,lahiri2003resampling}.
Since the exact cost depends on the resampling scheme and on how the horizon-$\tau$ pseudo-samples are represented, we keep $\mathcal C_{\mathrm{res}}(N,M,\tau)$ abstract in the subsequent complexity analysis.

Once the pseudo-samples have been generated, the empirical pseudo-CF $\widehat G_M^{\dagger}$ is formed exactly as in the direct-sampling setting, so the Fourier preprocessing cost remains $\mathcal O(MP)$ and the iterative training cost remains $\mathcal O(EPK)$. Hence the total one-dimensional complexity under resampling-based pseudo-sampling is
\begin{equation}
\label{eq:complexity_total_1d_boot}
\mathcal O\big(\mathcal C_{\mathrm{res}}(N,M,\tau)\big)
+
\mathcal O(MP)
+
\mathcal O(EPK).
\end{equation}

\paragraph{$d$-dimensional case.}
Let $\mathcal C_{\mathrm{res}}^{(d)}(N,M,\tau)$ denote the cost of generating pseudo-samples $\mathbf X_1^{\dagger},\ldots,\mathbf X_M^{\dagger}\in\Rbb^d$ by the resampling procedure. The empirical pseudo-CF construction then costs $\mathcal O(MPd)$, while training the full-covariance Gaussian--Laplace mixture costs $\mathcal O(EPKd^2)$. Therefore the total $d$-dimensional complexity under resampling-based pseudo-sampling is
\begin{equation}
\label{eq:complexity_total_d_boot}
\mathcal O\big(\mathcal C_{\mathrm{res}}^{(d)}(N,M,\tau)\big)
+
\mathcal O(MPd)
+
\mathcal O(EPKd^2).
\end{equation}

\subsection{Comparison with Expectation--Maximization}
Consider a 1D Gaussian-mixture model fitted by the Expectation--Maximization (EM) algorithm; see \cite{mclachlan2019finite}. If the model has $K$ components, then one expectation-step and one maximization-step over $M$ observations both scale linearly in $MK$, so the cost per iteration is $\mathcal O(MK)$. If $T_{\mathrm{EM}}$ iterations are required, the total cost of one fit is $\mathcal O(T_{\mathrm{EM}}MK)$. With $R$ random restarts and a search over several candidate values of $K$, the total cost grows as
\begin{equation}
\label{eq:complexity_em_search}
\mathcal O\big(R\,T_{\mathrm{EM}}\,M\,\sum_{K'\in\mathcal K}K'\big),
\end{equation}
where $\mathcal K$ is the set of tested model sizes.

\paragraph{Direct i.i.d.\ sampling setting.} In one-dimension, the essential difference between \eqref{eq:complexity_em_search} and \eqref{eq:complexity_total_1d} is that EM remains linear in $M$ at every iteration, whereas the Fourier-mixture method depends on $M$ only through the one-off construction of the empirical CF. After preprocessing, the iterative part of the algorithm scales only \mbox{with $P$ and $K$.}

For full-covariance Gaussian components in $d$ dimensions, one EM iteration scales as $\mathcal O(MKd^2)$, with an additional $\mathcal O(Kd^3)$ cost for covariance factorizations or inversions. Thus one full fit scales as
$\mathcal O\big(T_{\mathrm{EM}}(MKd^2+Kd^3)\big)$,
and with restarts and model-order search the total cost becomes
\[
\mathcal O\big(R\,T_{\mathrm{EM}}\sum_{K'\in\mathcal K}(MK'd^2+K'd^3)\big).
\]

\paragraph{Resampling-based pseudo-sampling setting.} A fair comparison should also include the cost of generating pseudo-samples. Thus, in one-dimension, EM trained on pseudo-samples scales as
\[
\mathcal O\big(\mathcal C_{\mathrm{res}}(N,M,\tau)\big)
+
\mathcal O\big(R\,T_{\mathrm{EM}}\,M\,\sum_{K'\in\mathcal K}K'\big),
\]
while in $d$ dimensions the corresponding full-covariance complexity becomes
\[
\mathcal O\big(\mathcal C_{\mathrm{res}}^{(d)}(N,M,\tau)\big)
+
\mathcal O\big(R\,T_{\mathrm{EM}}\sum_{K'\in\mathcal K}(MK'd^2+K'd^3)\big).
\]
By comparison, the Fourier-mixture method scales as \eqref{eq:complexity_total_1d_boot} in one-dimension and \eqref{eq:complexity_total_d_boot} in $d$-dimensions. Thus the qualitative distinction persists in both settings: once the empirical (pseudo-)CF has been formed, the iterative stage of the Fourier-mixture method is decoupled from the sample size $M$, whereas EM continues to scale linearly in $M$ at every iteration.

\section{Numerical experiments}
\label{sec:num}
This section is organized according to the two data settings developed in the theory. Subsection~\ref{subsec:num_direct} reports experiments under direct i.i.d.\ sampling from known target laws, including a controlled comparison of exact-CF and empirical-CF training, Gaussian-mixture sanity checks, non-Gaussian benchmarks, a two-dimensional example, and an error-decay study linked to the error analysis of Section~\ref{sec:error}. Subsection~\ref{subsec:num_pseudo} reports a resampling-based pseudo-sampling experiment built from Australian equity data, aligned with Section~\ref{sec:bootstrap}.


Unless otherwise stated, the one-dimensional experiments requiring model selection or held-out evaluation use a total sample size $M_{\mathrm{tot}}=10^6$, split as
$M_{\mathrm{tr}}:M_{\mathrm{val}}:M_{\mathrm{te}}=40\%:30\%:30\%$,
with $P=10^3$ Fourier training nodes on $[-\eta',\eta']=[-50,50]$ and minibatch size $1024$. The training set $\mathcal D_{\mathrm{tr}}$ is used to construct the empirical CF (or empirical pseudo-CF) and train the NN, $\mathcal D_{\mathrm{val}}$ for model selection and early stopping, and $\mathcal D_{\mathrm{te}}$ for out-of-sample evaluation. When this split is used, $M=M_{\mathrm{tr}}$; otherwise, $M$ denotes the full sample used to construct the empirical CF and is specified locally.

When the exact CF $G$ is available in closed form, we report
\[
L_2^{\operatorname{Re}}(G,\widehat G)
=
\int_{[-\eta',\eta']}
\big|\operatorname{Re}G(\eta)-\operatorname{Re}\widehat G(\eta)\big|^2\,d\eta,
\]
together with
\[
\operatorname{MPE}^{\operatorname{Re}}(G,\widehat G)
=
\max_{1\le q\le Q}
\big|\operatorname{Re}G(\eta_q)-\operatorname{Re}\widehat G(\eta_q)\big|,
\]
where $\{\eta_q\}_{q=1}^{Q}$ is a dense evaluation grid on $[-\eta',\eta']$, independent of the training grid; the imaginary-part quantities are defined analogously. When the exact density $g$ is available, we also report
\[
L_2(g,\widehat g)
=
\left(
\int_{[-A,A]} |g(x)-\widehat g(x)|^2\,dx
\right)^{1/2},
\]
for $A>0$ sufficiently large. When a test dataset is available, we report the per-sample negative log-likelihood
\[
\operatorname{NLL}(\widehat g;\mathcal D_{\mathrm{te}})
=
-\frac{1}{M_{\mathrm{te}}}\sum_{x\in\mathcal D_{\mathrm{te}}}\log \widehat g(x),
\]
so that smaller values indicate a better out-of-sample fit.

\subsection{Direct i.i.d.\ sampling}
\label{subsec:num_direct}

\subsubsection{Controlled comparison of exact-CF and empirical-CF training}
\label{subsec:num_exact_empirical}
We first quantify the additional estimation error caused by replacing the exact target CF with its empirical counterpart during Fourier-mixture NN training. The exact-CF reference model is trained against $G(\eta_p)$, whereas the empirical-CF model is trained against $\widehat G_M(\eta_p)$. The model class, Fourier nodes, loss function, optimizer, training budget, and initializations are held fixed, so the comparison isolates the effect of empirical-CF sampling. 

For the primary comparison, we consider the well-separated 3-GMM target
\[
g(x)=\sum_{j=1}^{3}\pi_j\,\varphi_{\mathcal G}(x;\mu_j,\sigma_j),
\qquad
G(\eta)=\sum_{j=1}^{3}\pi_j
\exp\!\left(i\mu_j\eta-\sigma_j^2\eta^2/2\right),
\]
with $(\pi_1,\pi_2,\pi_3)=(0.5,0.3,0.2)$, $(\mu_1,\mu_2,\mu_3)=(-4,0,4)$, and $(\sigma_1,\sigma_2,\sigma_3)=(1,1,1)$. Training samples are drawn i.i.d.\ from this target density, and all $M$ observations in each repetition are used to construct the empirical CF. We use the correctly specified Gaussian model $(K_{\mathcal G},K_{\mathcal L})=(3,0)$, set $[-\eta',\eta']=[-50,50]$, and take the $P=4000$ Fourier nodes to be the midpoints of a uniform partition of this interval. We consider $M\in\{10^3,4\times10^3,1.6\times10^4,6.4\times10^4,2.56\times10^5,10^6\}$. At each $M$, we use 30 independently generated datasets and the same 30 model initializations for the exact-CF and empirical-CF runs.

\begin{figure}[htbp]
\centering

\begin{minipage}[t]{0.55\textwidth}
\vspace{0pt}
\centering
\includegraphics[width=\linewidth]
{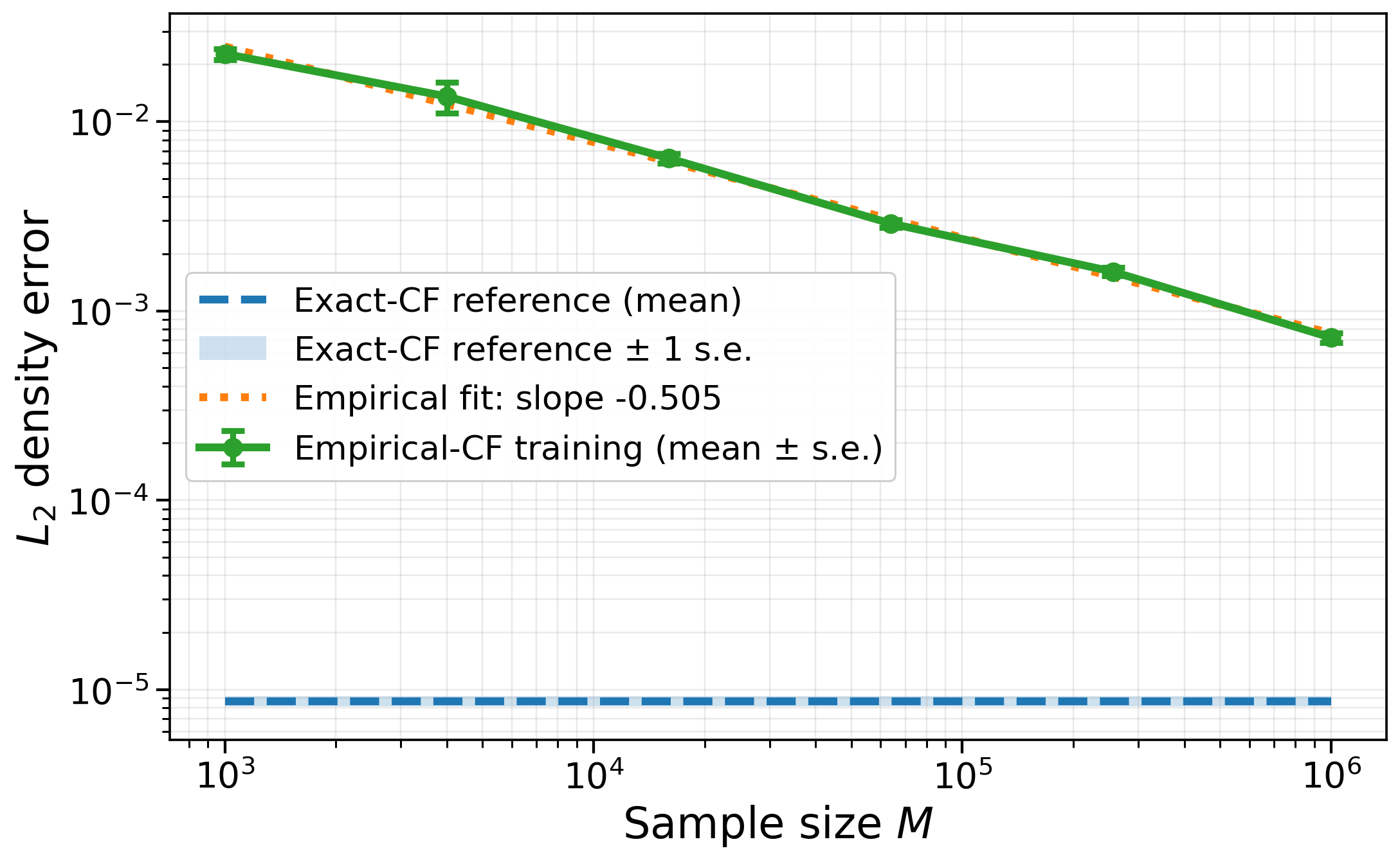}
\end{minipage}
\hfill
\begin{minipage}[t]{0.4\textwidth}
\vspace{0pt}
\raggedright
\begingroup
\setlength{\abovecaptionskip}{0pt}
\captionof{figure}{Exact-CF reference versus empirical-CF training
for the well-separated 3-GMM target using deterministic midpoint nodes. The empirical-CF curve shows the mean
density $L_2$ error over 30 independent repetitions, with error bars equal to
one standard error; the dashed line and shaded band show the mean exact-CF reference and one standard error, respectively.}
\label{fig:gmm_exact_vs_empirical}
\endgroup
\end{minipage}

\end{figure}

The exact-CF reference attains a mean density $L_2$ error of approximately $8.66\times10^{-6}$. By comparison, the mean empirical-CF error decreases from $2.28\times10^{-2}$ at $M=10^3$ to $6.40\times10^{-3}$ at $M=1.6\times10^4$ and $7.22\times10^{-4}$ at $M=10^6$. The corresponding additional errors above the exact-CF reference are approximately $2.28\times10^{-2}$, $6.39\times10^{-3}$, and $7.13\times10^{-4}$. The fitted log--log slopes are $-0.505$ for the empirical-CF error and $-0.507$ for the additional error, with $R^2=0.997$ in both cases. Thus, over the tested range, the error induced by replacing $G$ with $\widehat G_M$ dominates the much smaller exact-CF reference error and decreases systematically as $M$ increases.

As secondary evidence, we consider the short-horizon Kou jump--diffusion benchmark from \cite{kou01}. The increment over horizon $T>0$ is
\[
X_T
=
T\big(\mu-\lambda\kappa-\tfrac{\sigma^2}{2}\big)
+
\sigma\sqrt{T}\,Z
+
\sum_{j=1}^{N_T}Y_j,
\]
where $Z\sim N(0,1)$, $N_T\sim\mathrm{Poisson}(\lambda T)$, and the jumps are independently distributed as $Y_j=U_j^+$ with probability $p$ and $Y_j=-U_j^-$ with probability $1-p$, with $U_j^+\sim\mathrm{Exp}(\zeta_1)$ and $U_j^-\sim\mathrm{Exp}(\zeta_2)$. Here
\[
\kappa
=
p\,\frac{\zeta_1}{\zeta_1-1}
+
(1-p)\,\frac{\zeta_2}{\zeta_2+1}
-1,
\]
and the exact CF is
\[
G_{\mathrm{Kou}}(\eta)
=
\exp\!\Big(
T\Big[
i\eta\Big(\mu-\lambda\kappa-\frac{\sigma^2}{2}\Big)
-\frac{\sigma^2\eta^2}{2}
+
\lambda\Big(
p\,\frac{\zeta_1}{\zeta_1-i\eta}
+
(1-p)\,\frac{\zeta_2}{\zeta_2+i\eta}
-1
\Big)
\Big]
\Big).
\]
We take $T=10^{-3}$, $\mu=0.05$, $\sigma=0.15$, $\lambda=0.1$, $p=0.3445$, $\zeta_1=3.0465$, and $\zeta_2=3.0775$, following \cite[Section~6.5]{du2025fourier}. Training samples for the empirical-CF model are generated from the increment representation above. Both models use $(K_{\mathcal G},K_{\mathcal L})=(45,0)$, a common $P=10{,}001$ Fourier grid, and five paired initializations. Owing to greater optimization variability, we report medians. The exact-CF reference has density $L_2$ error $6.54\times10^{-3}$, whereas the empirical-CF error decreases from $2.43\times10^{-1}$ at $M=10^3$ to $8.49\times10^{-2}$, $2.21\times10^{-2}$, and $1.07\times10^{-2}$ at $M=10^4$, $10^5$, and $10^6$, respectively. Thus empirical-CF training again approaches exact-CF performance as $M$ increases, with a remaining gap of approximately $4.16\times10^{-3}$ at $M=10^6$.

\subsubsection{Gaussian-mixture sanity checks}
\label{subsec:num_gmm}
For the synthetic Gaussian-mixture benchmarks, the data are generated directly from the target densities by standard mixture sampling: a component label is drawn from the categorical distribution determined by the mixture weights, and then a sample is drawn from the corresponding Gaussian component. This yields i.i.d.\ observations from the target density and introduces no additional approximation error at the data-generation stage.

We retain the well-separated 3-GMM target introduced in Subsection~\ref{subsec:num_exact_empirical} and add an overlapping three-component Gaussian mixture with the same weights given by $(\pi_1,\pi_2,\pi_3)=(0.5,0.3,0.2)$, means $(\mu_1,\mu_2,\mu_3)=(-2,0,2)$, and standard deviations $(\sigma_1,\sigma_2,\sigma_3)=(1,1,2)$.

For these tests, we restrict the learned model to the Gaussian subclass by setting $K_{\mathcal L}=0$ and varying $K_{\mathcal G}$. Figure~\ref{fig:num_gmm_fournet_capacity} shows the effect of $K_{\mathcal G}$ on the density $L_2$ error, per-sample NLL, and training time. In both targets, the best mean density accuracy is attained at $K_{\mathcal G}=3$, matching the number of components in the ground truth. Increasing the model size beyond this value does not improve the fit: the average $L_2$ error increases gradually, while the NLL remains nearly flat. This is the expected behaviour for a positive finite mixture trained on a well-specified target.

\begin{figure}[htbp]
\centering
\begin{subfigure}[t]{0.48\textwidth}
\centering
\includegraphics[width=\textwidth]{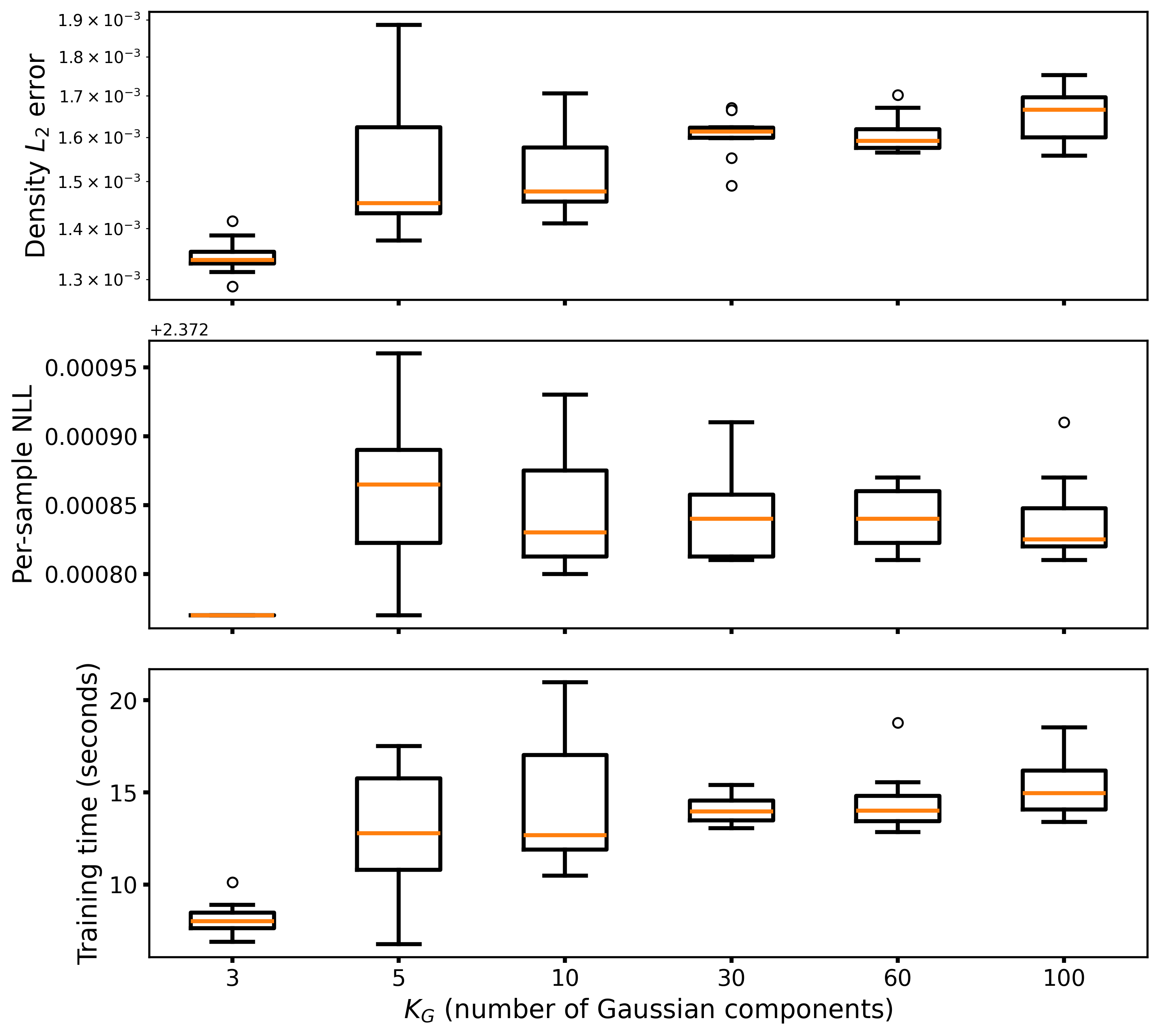}
\caption{Well-separated components}
\end{subfigure}
\hfill
\begin{subfigure}[t]{0.48\textwidth}
\centering
\includegraphics[width=\textwidth]{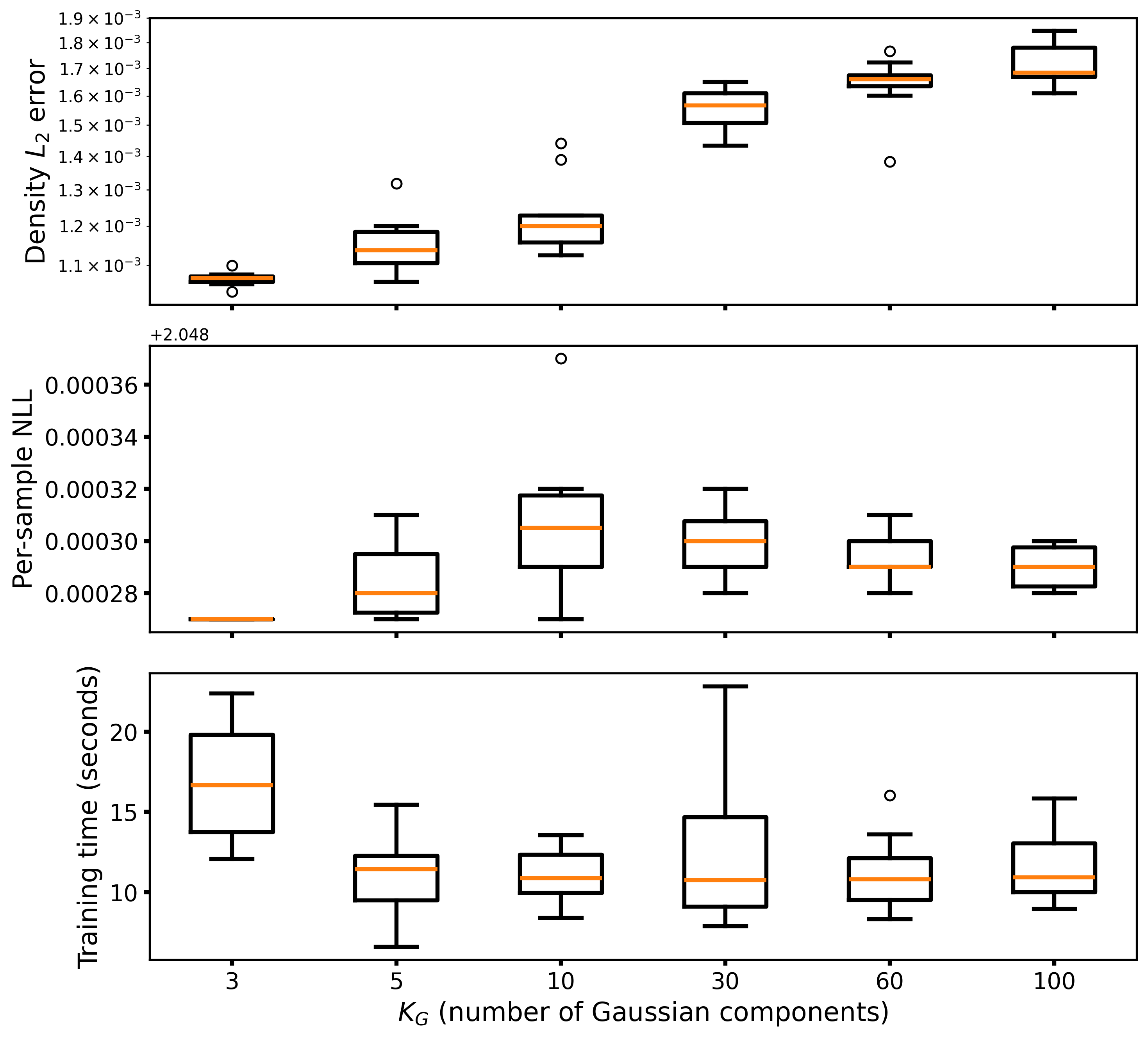}
\caption{Overlapping components}
\end{subfigure}
\caption{Effect of the number $K_{\mathcal G}$ of learned Gaussian components on the density $L_2$ error, per-sample NLL, and training time. In both examples, the best mean density accuracy is attained at $K_{\mathcal G}=3$.}
\label{fig:num_gmm_fournet_capacity}
\end{figure}

Table~\ref{tab:num_gmm_em_compare} compares the data-driven Fourier-mixture method with a capacity-matched Gaussian-mixture model fitted by EM, with $K_{\mathcal G}=K_{\mathrm{EM}}=3$.

\begin{table}[htbp]
\centering
\caption{Capacity-matched comparison between the data-driven Fourier-mixture method (Gaussian subclass) and EM on the three-component Gaussian-mixture targets.}
\label{tab:num_gmm_em_compare}
\begin{tabular}{l l l l l}
\toprule
Method & Target & Components & $L_2(g,\widehat g)$ & NLL \\
\midrule
Fourier-mixture & Well-separated & $K_{\mathcal G}=3$ & $1.32\times 10^{-3}$ & $2.373$ \\
EM & Well-separated & $K_{\mathrm{EM}}=3$ & $1.07\times 10^{-3}$ & $2.370$ \\
\midrule
Fourier-mixture & Overlapping & $K_{\mathcal G}=3$ & $1.05\times 10^{-3}$ & $2.048$ \\
EM & Overlapping & $K_{\mathrm{EM}}=3$ & $2.29\times 10^{-3}$ & $2.048$ \\
\bottomrule
\end{tabular}
\end{table}

In the well-separated case, the two methods are essentially comparable in both NLL and density $L_2$ error, with EM slightly better in $L_2$. In the overlapping case, however, the proposed method attains a substantially smaller density $L_2$ error while maintaining essentially the same NLL. Thus the method behaves correctly on standard Gaussian-mixture benchmarks and remains competitive with EM in the classical Gaussian setting.

\subsubsection{Heavy-tail and jump examples}
\label{subsec:num_heavy}
We next consider the Cauchy distribution, a canonical heavy-tailed benchmark for which both the density and the CF are available in closed form. Together with the short-horizon Kou jump benchmark considered in Subsection~\ref{subsec:num_exact_empirical}, this allows us to assess the method on both heavy-tailed and jump-driven non-Gaussian laws.

For the Cauchy test, the training data are generated directly from the Cauchy distribution with location $x_0=0$ and scale $\gamma=0.5$, whose density and CF are
\[
g_{\mathrm{Cau}}(x)
=
\frac{1}{\pi\gamma}\frac{1}{1+\big((x-x_0)/\gamma\big)^2},
\qquad
G_{\mathrm{Cau}}(\eta)
=
\exp(i x_0\eta-\gamma|\eta|).
\]
The Gaussian-mixture model corresponds to $(K_{\mathcal G},K_{\mathcal L})=(10,0)$, while the Gaussian--Laplace mixture uses $(K_{\mathcal G},K_{\mathcal L})=(5,5)$ and is trained against the empirical CF computed from the same sample. The exact CF is used only as an analytic reference for evaluation.

\begin{table}[htbp]
\centering
\caption{Fourier-domain errors between $G$ and $\widehat G$ for the Cauchy example.}
\label{tab:num_cauchy_kou}
\begin{tabular}{l l l l l}
\toprule
Model (mixture) & $L_2^{\operatorname{Re}}(\cdot)$ & $L_2^{\operatorname{Im}}(\cdot)$ & $\operatorname{MPE}^{\operatorname{Re}}(\cdot)$ & $\operatorname{MPE}^{\operatorname{Im}}(\cdot)$\\
\midrule
Cauchy (Gaussian) & $1.3\times 10^{-6}$ & $4.2\times 10^{-7}$ & $2.0\times 10^{-2}$ & $1.6\times 10^{-3}$\\
Cauchy (Gaussian--Laplace) & $1.2\times 10^{-8}$ & $3.0\times 10^{-9}$ & $5.0\times 10^{-3}$ & $7.2\times 10^{-4}$\\
\bottomrule
\end{tabular}
\end{table}

Table~\ref{tab:num_cauchy_kou} shows that adding Laplace components reduces the Fourier-domain error by one to two orders of magnitude for the Cauchy target in both the real and imaginary parts. The controlled Kou comparison in Subsection~\ref{subsec:num_exact_empirical} provides complementary evidence for a jump-driven law: empirical-CF training approaches the exact-CF reference as $M$ increases. Taken together, the two examples show that the Gaussian--Laplace extension is most beneficial when the target exhibits heavier tails or sharper local structure, while empirical-CF Fourier training remains effective for the short-horizon Kou benchmark.

\subsubsection{Two-dimensional Cauchy example}
\label{subsec:num_multi}
To demonstrate that the multivariate extension is more than a formal generalization, we consider a genuinely heavy-tailed 2D benchmark. In the present test, the target is the product of two independent centered Cauchy marginals with common scale $\gamma=0.5$, so that
\[
g_{\mathrm{Cau},2}(\mathbf{x})
=
\prod_{j=1}^{2}
\frac{1}{\pi\gamma}\frac{1}{1+\big(x_j/\gamma\big)^2},
\qquad
G_{\mathrm{Cau},2}(\boldsymbol{\eta})
=
\exp\!\big(-\gamma(|\eta_1|+|\eta_2|)\big),
\]
for $\mathbf{x}=(x_1,x_2)\in\Rbb^2$ and $\boldsymbol{\eta}=(\eta_1,\eta_2)\in\Rbb^2$. The training data are generated directly from this product-Cauchy law and then used to form the empirical CF. We compare three learned models: two Gaussian-mixture models with $(K_{\mathcal G},K_{\mathcal L})=(30,0)$ and $(40,0)$, and one Gaussian--Laplace mixture with $(K_{\mathcal G},K_{\mathcal L})=(30,20)$.

\begin{table}[htbp]
\centering
\caption{Fourier-domain errors between $G$ and $\widehat G$ for the 2D Cauchy example.}
\label{tab:num_multi_cauchy}
\setlength{\tabcolsep}{3.8pt}
\begin{tabular}{l c c c c c}
\toprule
Model & $(K_{\mathcal G},K_{\mathcal L})$ & $L_2^{\operatorname{Re}}(\cdot)$ & $L_2^{\operatorname{Im}}(\cdot)$ & $\operatorname{MPE}^{\operatorname{Re}}(\cdot)$ & $\operatorname{MPE}^{\operatorname{Im}}(\cdot)$\\
\midrule
Gauss. & (30,0) & $3.30\times 10^{-2}$ & $1.59\times 10^{-3}$ & $2.42\times 10^{-2}$ & $1.76\times 10^{-3}$\\
Gauss. & (40,0) & $4.65\times 10^{-2}$ & $1.01\times 10^{-3}$ & $2.83\times 10^{-2}$ & $1.29\times 10^{-3}$\\
Gauss.--Laplace & (30,20) & $1.51\times 10^{-2}$ & $5.36\times 10^{-4}$ & $1.20\times 10^{-2}$ & $2.93\times 10^{-4}$\\
\bottomrule
\end{tabular}
\end{table}

Table~\ref{tab:num_multi_cauchy} shows that the Gaussian--Laplace model improves all reported CF error metrics relative to the Gaussian-only baselines. Enlarging the Gaussian-only model from $(30,0)$ to $(40,0)$ improves some imaginary-part errors but does not improve the real-part errors, whereas adding Laplace components improves both. Thus the Fourier-domain training strategy remains effective in two dimensions, and the Gaussian--Laplace model continues to offer a clear advantage on a genuinely heavy-tailed target.

\subsubsection{$L_2$ error decay study}
\label{subsec:num_scaling}
We conclude the direct i.i.d.\ experiments with an $L_2$ error-decay study motivated by Theorem~\ref{thm:error_decomposition_section3} and Corollary~\ref{cor:expected_L2_norm_section3}, using the well-separated 3-GMM target and the correctly specified Gaussian model $(K_{\mathcal G},K_{\mathcal L})=(3,0)$ introduced in Subsection~\ref{subsec:num_exact_empirical}.
The model is therefore well specified. For both studies, the Fourier nodes are the midpoints of a uniform partition of $[-50,50]$. We fix $\eta'=50$, use $30$ independent repetitions per setting, take $P=4000$ for the $M$-dependence study, {\sblue{where $M$ is the number of observations entering the empirical CF,}} and vary
\[
M\in\{10^3,\;4\times10^3,\;1.6\times10^4,\;6.4\times10^4,\;2.56\times10^5,\;10^6\}.
\]
For the $P$-dependence study, we fix $M=10^6$ and vary
\[
P\in\{250,\;500,\;1000,\;2000,\;4000,\;8000,\;16{,}000\}.
\]
The theory-linked quantities $\mathrm{MSE}_P^{*}$, $R_P^{*}$, $\mathrm{Loss}_P^{*}$, $V_P$, and $W_P$ are also recorded.
\begin{figure}[htbp]
\centering
\begin{subfigure}[t]{0.48\textwidth}
\centering
\includegraphics[width=\textwidth]{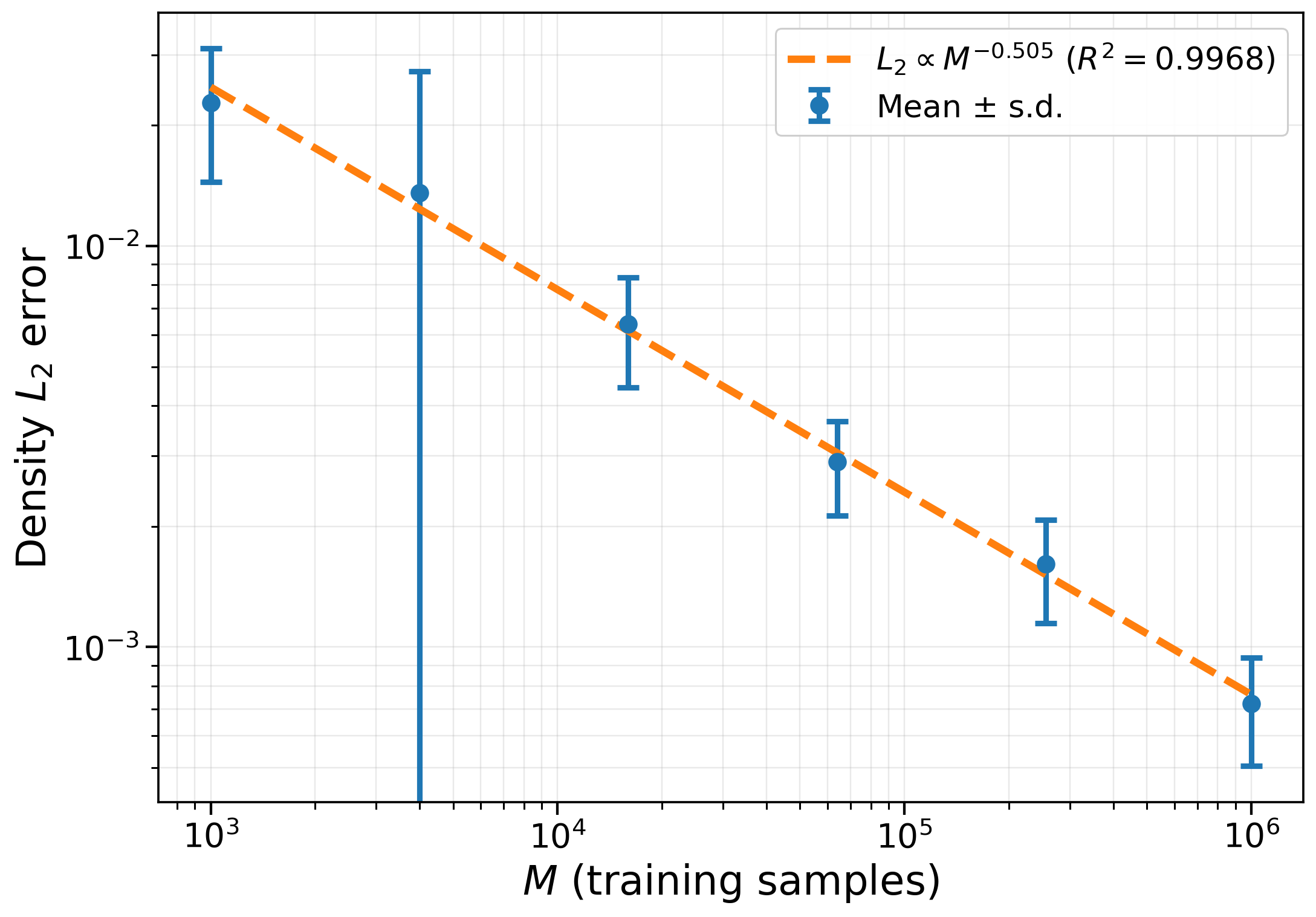}
\caption{$M$-dependence with $P=4000$}
\end{subfigure}
\hfill
\begin{subfigure}[t]{0.48\textwidth}
\centering
\includegraphics[width=\textwidth]{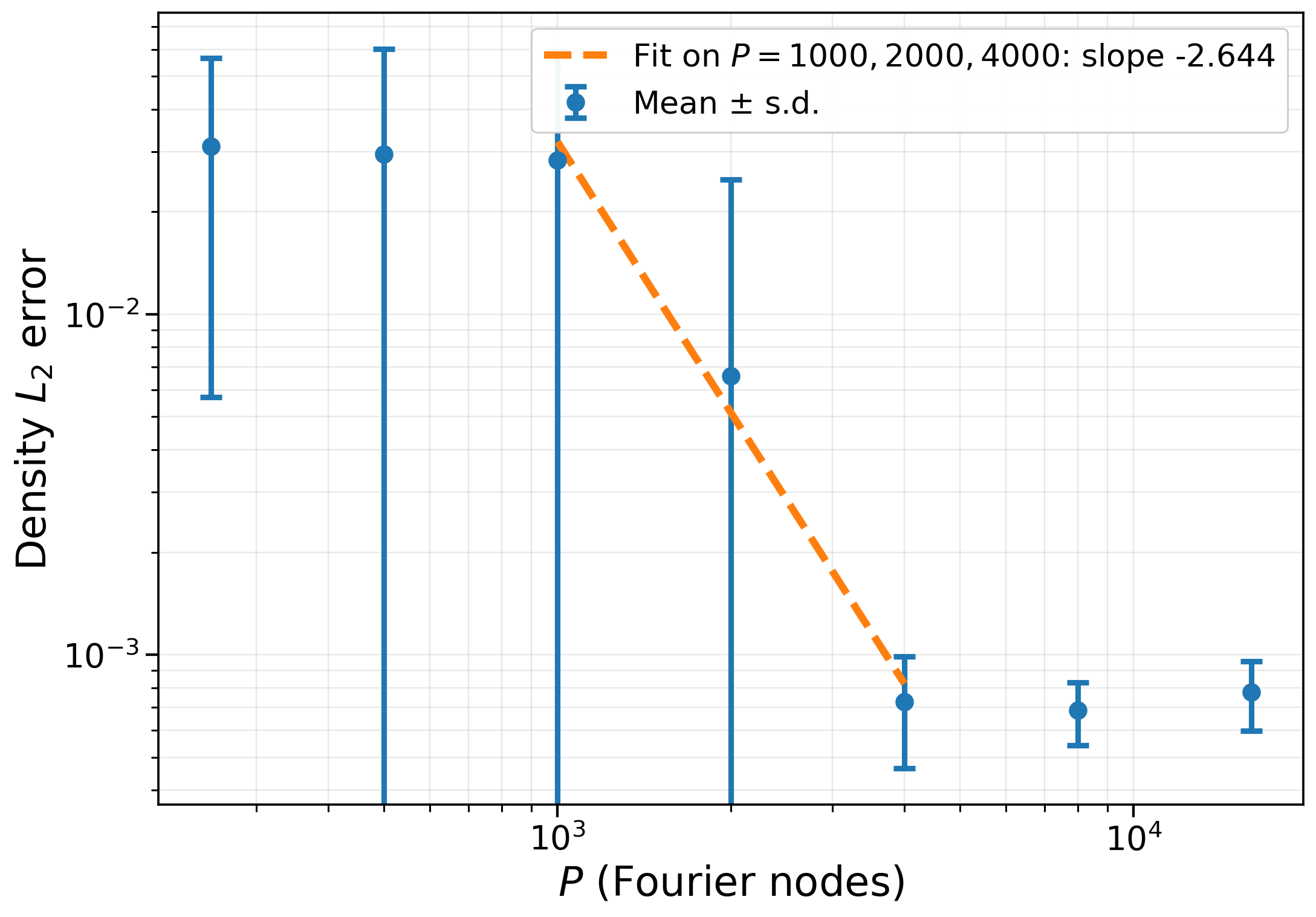}
\caption{$P$-dependence with $M=10^6$}
\end{subfigure}
\caption{$L_2$ error decay for the well-specified 3-GMM target.}
\label{fig:primary_scaling}
\end{figure}
The fitted log--log slope for the $M$-dependence study is $-0.505$ with $R^2=0.997$, so the observed mean $L_2$ error decays at a rate close to $M^{-1/2}$. This matches the order of the empirical-CF contribution to the expected $L_2$ norm identified in Corollary~\ref{cor:expected_L2_norm_section3}; the corollary does not, by itself, assert a universal rate for the total error because the truncation, achieved-loss, and discretization terms also contribute. The exact-CF reference in Figure~\ref{fig:gmm_exact_vs_empirical} is more than eighty times smaller than the empirical-CF error at $M=10^6$, supporting the interpretation that empirical-CF sampling is the dominant source of error over the tested range. The $P$-dependence curve shows a steep initial decrease followed by saturation, indicating that once the Fourier grid is sufficiently fine, further refinement has little effect relative to the remaining finite-sample and training errors.

As a secondary check, we also repeated the $L_2$ error-decay experiment for the Cauchy target from Subsection~\ref{subsec:num_heavy}, with $x_0=0$ and $\gamma=0.5$ (results not shown). This case is intentionally misspecified: a fixed finite Gaussian--Laplace mixture cannot represent the Cauchy density exactly, so the total error is expected to contain an approximation floor in addition to the empirical-CF sampling and discretization errors. This is precisely what is observed. For example, at $M=10^6$ and $P=4000$, increasing the model size from $(K_{\mathcal G},K_{\mathcal L})=(3,3)$ to $(100,30)$ and $(150,45)$ reduces the realized $L_2$ error floor from about $1.82\times 10^{-2}$ to $9.04\times 10^{-3}$ and $8.80\times 10^{-3}$, respectively. The fitted transient $M$-dependence also moves closer to the well-specified $M^{-1/2}$ behavior, with the log-log slope changing from about $-0.642$ for $(K_{\mathcal G},K_{\mathcal L})=(3,3)$ to about $-0.604$ for $(K_{\mathcal G},K_{\mathcal L})=(150,45)$. Thus the Cauchy study illustrates the expected bias--variance picture: increasing model capacity reduces the approximation floor, while the well-specified 3-GMM study gives the clearest observation of approximately $M^{-1/2}$ decay in the density $L_2$ error over the tested range.

\subsection{Resampling-based pseudo-sampling}
\label{subsec:num_pseudo}
\subsubsection{Data and processing}
\label{sssec:num_boot_data}
The resampling-based pseudo-sampling experiment is built from a long monthly real return series for Australian equity, expressed in AUD. As a proxy for the Australian equity market, we use a capitalization-weighted total return series for the All Ordinaries index. Monthly price levels are taken from the Bloomberg All Ordinaries Price Index (AS30). Prior to the start of the Bloomberg All Ordinaries Accumulation Index in 1999, reinvested dividends are imputed using trailing dividend-yield data from \cite{Mathews2019AustEquities} and the Reserve Bank of Australia. The nominal index is then deflated by the Australian CPI, yielding monthly real Australian equity returns over 1935{:}1--2022{:}12.

The target horizon in this experiment is one year. Starting from the monthly real return series, we generate many independent 30-year bootstrap paths using the stationary bootstrap of \cite{politis1994stationary}, with geometrically distributed block lengths having expected length $24$ months
From each path, we extract one fixed annual observation and use the resulting one-year total returns
{\sblue{$X_1^{\dagger},\ldots,X_{M_{\mathrm{tot}}}^{\dagger}$, with $M_{\mathrm{tot}}=2.56\times 10^5$.}}
{\sblue{After the split described below, the $M=M_{\mathrm{tr}}$ observations in $\mathcal D_{\mathrm{tr}}$ are used to construct the empirical pseudo-CF}}
$\widehat G_M^{\dagger}(\eta)
=
\frac{1}{M}\sum_{m=1}^{M} e^{i\eta X_m^{\dagger}}$.
Thus the experiment is based on a resampling law for one-year Australian equity returns rather than on directly observed i.i.d.\ annual data.

\subsubsection{Australian equity pseudo-sample experiment}
\label{subsec:num_au_equity}
We use a train/validation/test split
$M_{\mathrm{tr}}:M_{\mathrm{val}}:M_{\mathrm{te}}
=
102{,}400:76{,}800:76{,}800$,
{\sblue{so that $M=M_{\mathrm{tr}}=102{,}400$ pseudo-samples are used to construct $\widehat G_M^{\dagger}$,}}
$P=1000$ uniformly spaced Fourier nodes on $[-50,50]$, and the same training schedule as in the preceding experiments. Model sizes are chosen by validation NLL.
For EM-GMM, however, naive minimum-NLL selection is unreliable in this pseudo-sample setting: because the resampled dataset contains repeated values, the likelihood can continue to improve as Gaussian components shrink around discrete atoms. We therefore test EM-GMM component counts across a range of $K_{\mathrm{EM}}$ and inspect the validation-NLL and marginal-improvement profiles in Figure~\ref{fig:au_em_model_selection}. The validation NLL continues to improve for larger models, but the marginal gain drops sharply after $K_{\mathrm{EM}}=8$; beyond this point, additional components mainly add local flexibility around repeated pseudo-sample values. We therefore use the parsimonious EM-GMM benchmark with $K_{\mathrm{EM}}=8$ as the main comparison.

\begin{figure}[htbp]
    \centering
    \begin{subfigure}[b]{0.95\textwidth}
        \centering
        \includegraphics[width=\textwidth]{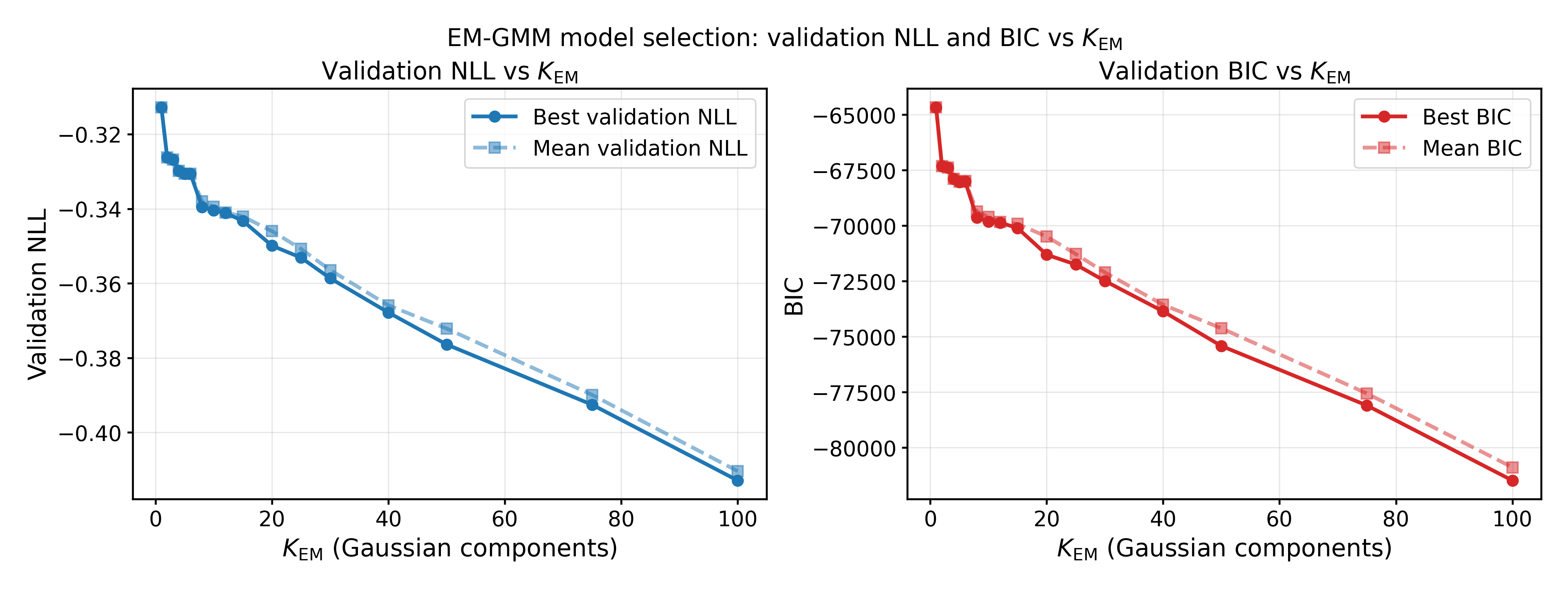}
        \caption{Validation NLL and BIC.}
    \end{subfigure}
 \vspace{0.75em}
    \begin{subfigure}[b]{0.95\textwidth}
        \centering
        \includegraphics[width=\textwidth]{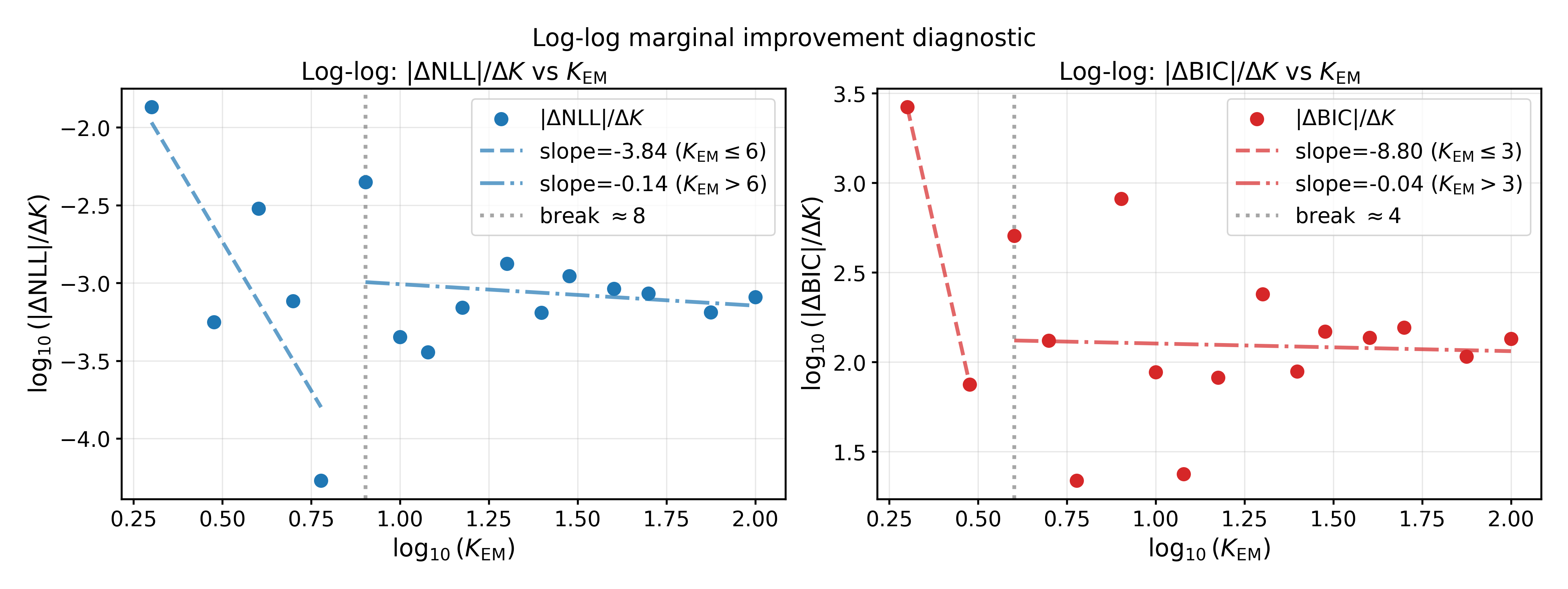}
        \caption{Marginal improvement.}
    \end{subfigure}
    \caption{EM-GMM model selection for the Australian equity pseudo-sample. The diagnostics identify $K_{\mathrm{EM}}=8$ as a parsimonious benchmark size.}
    \label{fig:au_em_model_selection}
\end{figure}

\begin{table}[htbp]
  \centering
  \caption{Selected models and test performance for the Australian equity pseudo-sample. The upper panel reports model size and NLL; the lower panel reports Fourier-domain errors against the empirical test pseudo-CF.}
  \label{tab:au_pseudosample_summary_updated}
  \begin{tabular}{l c c c}
    \toprule
    Method & $(K_{\mathcal G},K_{\mathcal L})$ & Val NLL & Test NLL \\
    \midrule
    EM-GMM       & $(8,0)$   & $-0.3308$ & $-0.3250$ \\
    Fourier-GMM  & $(15,0)$  & $-0.3329$ & $-0.3269$ \\
    Fourier-GLMM & $(15,6)$  & $-0.3344$ & $-0.3286$ \\
    \bottomrule
  \end{tabular}

  \vspace{1ex}

  \begin{tabular}{l c c c c}
    \toprule
    Method & $L_2^{\operatorname{Re}}(\cdot)$ & $L_2^{\operatorname{Im}}(\cdot)$ & $\operatorname{MPE}^{\operatorname{Re}}(\cdot)$ & $\operatorname{MPE}^{\operatorname{Im}}(\cdot)$ \\
    \midrule
    EM-GMM       & $4.194\times10^{-3}$ & $4.723\times10^{-3}$ & $2.393\times10^{-2}$ & $2.632\times10^{-2}$ \\
    Fourier-GMM  & $1.446\times10^{-3}$ & $1.409\times10^{-3}$ & $1.137\times10^{-2}$ & $9.414\times10^{-3}$ \\
    Fourier-GLMM & $1.359\times10^{-3}$ & $1.375\times10^{-3}$ & $8.886\times10^{-3}$ & $8.832\times10^{-3}$ \\
    \bottomrule
  \end{tabular}
\end{table}

\begin{figure}[!htbp]
    \centering
    \includegraphics[width=1\textwidth]{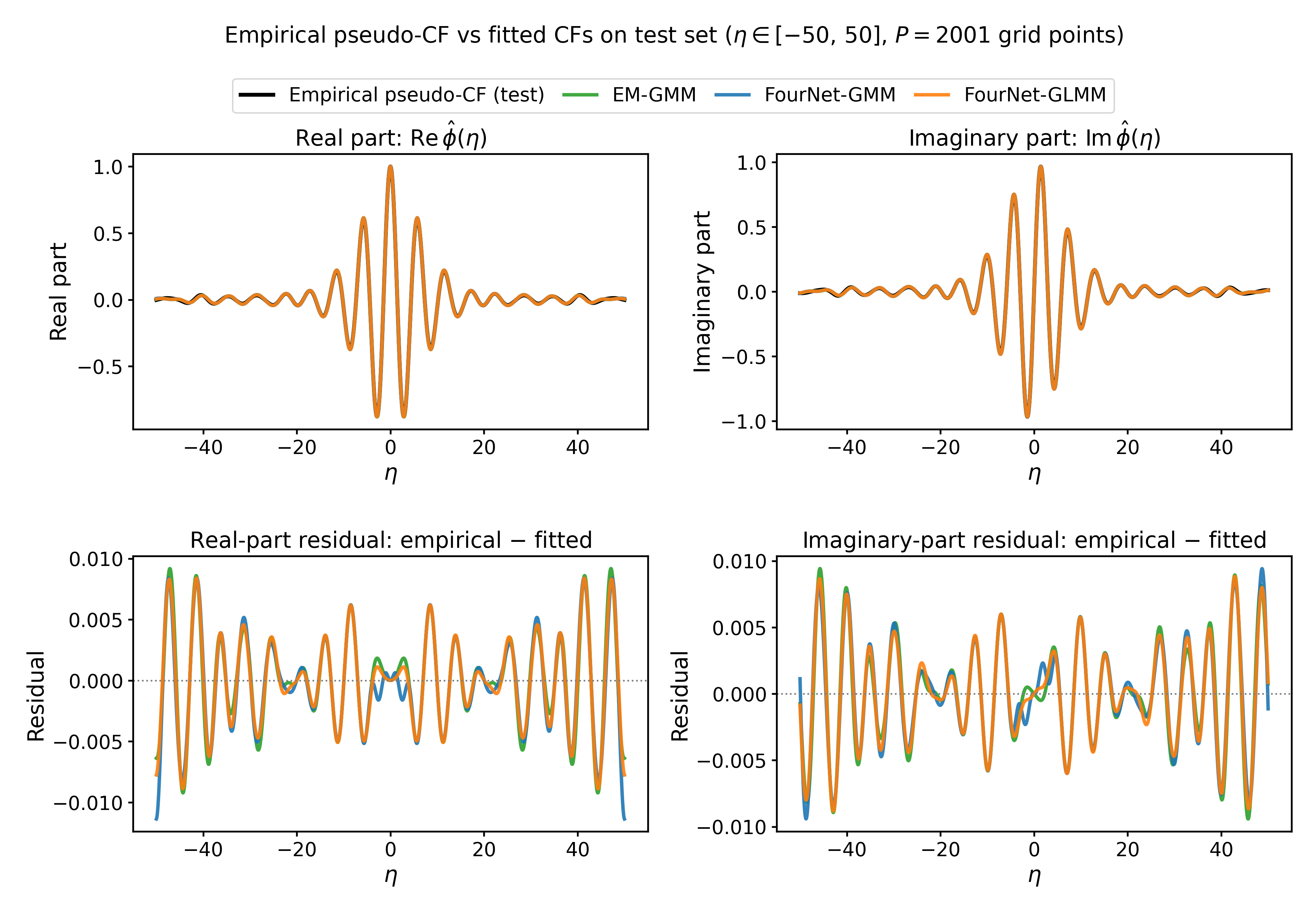}
    \caption{Australian equity pseudo-sample experiment: empirical test pseudo-CF, fitted CFs, and residual curves under the parsimonious EM-GMM benchmark.}
    \label{fig:au_equity_main_results}
\end{figure}

Table~\ref{tab:au_pseudosample_summary_updated} shows that Fourier-GLMM achieves the best validation and test NLL under the parsimonious EM comparison, while also giving the smallest real-part and imaginary-part CF errors. The advantage over EM is clearest in Fourier space, which is consistent with the fact that the Fourier models are trained against CF data rather than directly by maximum likelihood. Figure~\ref{fig:au_equity_main_results} confirms this visually: the Fourier-based fitted CFs track the empirical test pseudo-CF more closely than the EM benchmark over the diagnostic frequency range.

Since Fourier-domain accuracy and NLL do not by themselves guarantee accurate tail behavior, and since one motivation for the Gaussian--Laplace model is improved flexibility for non-Gaussian tails, we also examine tail calibration. Tail calibration is adequate for all three methods, with no uniform dominance across tail levels. At the deepest lower tail ($\alpha=0.5\%$), EM-GMM attains the smallest lower-tail probability error ($1.6\times10^{-4}$), whereas Fourier-GLMM is most accurate in the upper tail ($1.2\times10^{-4}$). At more moderate levels, the methods remain close. This is consistent with the modest excess kurtosis of the pseudo-sample: the train, validation, and test splits have excess kurtosis $0.743$, $0.735$, and $0.723$, respectively. Since these values are close to the Gaussian benchmark of $0$ and well below the Laplace value $3$, the one-year total-return distribution is only moderately heavy-tailed, so the Laplace part of the GLMM improves the Fourier-domain fit without producing a uniformly dominant tail-calibration result.

\section{Conclusion and future work}
\label{sec:conclusion}
We introduced a data-driven Fourier-trained neural-network method for estimating fixed-horizon probability densities from empirical transform information. The method combines Fourier-domain training with a positive Gaussian--Laplace output structure, so the learned model remains a proper density throughout optimization while retaining a closed-form CF. Thus, the estimation problem can be formulated directly in Fourier space rather than through physical-space likelihood fitting alone.

The analysis was developed for two sampling settings. In the direct i.i.d.\ case, we derived an {\dpurple{expected squared $L_2$ density-error bound, together with a corresponding bound for the expected $L_2$ norm,}} separating Fourier truncation, empirical training error, discretization, and empirical-CF sampling error. In the resampling-based pseudo-sampling case, the same Fourier-to-real-space argument carries over conditionally on the observed history, with {\dpurple{an additional pseudo-law discrepancy term}} measuring the gap between the resampling law and the true fixed-horizon law. We also extended the framework to multiple dimensions and discussed its computational complexity.

The numerical experiments support the main conclusions. The method is competitive with EM on well-specified Gaussian-mixture benchmarks, improves over the Gaussian subclass on heavy-tailed targets such as the Cauchy law, and a controlled exact-CF reference comparison quantifies the additional density error introduced by empirical-CF training. In the well-specified 3-GMM study, the observed $L_2$ norm decays at a rate close to the inverse square root of the sample size when the remaining error terms are controlled, consistent with the bound for the expected $L_2$ norm. The method also remains viable in two dimensions. The Australian equity experiment further illustrates that the resampling-based framework can be used to estimate fixed-horizon laws from dependent time-series data.

Several directions merit further study. On the theoretical side, it would be useful to derive explicit convergence rates for the pseudo-law discrepancy terms under specific resampling schemes, and to extend the analysis beyond conditionally i.i.d.\ pseudo-samples to more general dependent settings. On the application side, the resampling-based pseudo-sampling framework provides a route toward data-driven fixed-horizon laws in stochastic control and portfolio optimization, complementing existing Fourier-trained transition-kernel approaches based on closed-form CFs \cite{dang2026monotone,dang2026multi}.

\vfill

\appendix
\section{Proof of Lemma~\ref{lem:discretization_section3}}
\label{app:discretization_section3}
By Assumption~\ref{ass:modelling}~(A3), $G$ is Lipschitz on $[-\eta',\eta']$. Under \eqref{eq:Theta_GL_section2}, the real and imaginary parts of $\widehat G(\cdot;\theta)$ given by \eqref{eq:GL_cf_real_section2}--\eqref{eq:GL_cf_imag_section2} are uniformly Lipschitz on $[-\eta',\eta']$, and hence so is $F_\theta(\cdot)$. Thus
\[
|F_\theta(\eta)-F_\theta(\xi)|\le C'|\eta-\xi|,
\qquad
\forall\,\eta,\xi\in[-\eta',\eta'],\ \forall\,\theta\in\Theta,
\]
for some $C'=C'(\eta')>0$. {\dpurple{For each cell $[\xi_{p-1},\xi_p]$ and its associated tag $\eta_p\in[\xi_{p-1},\xi_p]$, we have
\[
\begin{aligned}
\left|
\delta_p F_\theta(\eta_p)
-
\int_{\xi_{p-1}}^{\xi_p}F_\theta(\eta)\,d\eta
\right|
&\le
\int_{\xi_{p-1}}^{\xi_p}
|F_\theta(\eta_p)-F_\theta(\eta)|\,d\eta
\\
&\le
C'\delta_p^2.
\end{aligned}
\]
Summing over $p=1,\ldots,P$ and using
$\delta_p\le\delta_{\max}=C_1/P$ gives
\[
\left|
\sum_{p=1}^{P}\delta_p F_\theta(\eta_p)
-
\int_{-\eta'}^{\eta'}F_\theta(\eta)\,d\eta
\right|
\le
C'\sum_{p=1}^{P}\delta_p^2
\le
C'P\delta_{\max}^2
=
\frac{C'C_1^2}{P},
\]
which proves \eqref{eq:disc_bound_section3}.}}

\section{Proof of Lemma~\ref{lem:truncation_section3}}
\label{app:truncation_section3}
By Assumption~\ref{ass:modelling}~(A2) and Plancherel's theorem, $G\in L_2(\Rbb)$, which gives the first bound in  \eqref{eq:trunc_section3} for $\eta'>0$ sufficiently large. For the second, \eqref{eq:GL_cf_section2} and \eqref{eq:Theta_GL_section2} imply
\[
|\widehat G(\eta;\theta)|
\le
\sum_{k=1}^{K_{\mathcal G}}\beta_k^{(\mathcal G)} \exp(-\sigma_k^2\eta^2/2)
+
\sum_{\ell=1}^{K_{\mathcal L}}\beta_\ell^{(\mathcal L)}\frac{1}{1+b_\ell^2\eta^2}
\le
\exp(-\sigma_{\min}^2\eta^2/2)
+
\frac{1}{1+b_{\min}^2\eta^2},
\]
whose square is integrable on $\Rbb$, uniformly in $\theta$.
The  bound \eqref{eq:tail_bound_section3} then follows~from
\[
|G(\eta)-\widehat G(\eta;\theta)|^2
\le
2|G(\eta)|^2+2|\widehat G(\eta;\theta)|^2.
\]

\section{Proof of  Lemma~\ref{lem:bootstrap_ecf}}
\label{app:bootstrap_ecf}
Given $\mathcal H_N$, the pseudo-samples $X_1^{\dagger},\ldots,X_M^{\dagger}$ are i.i.d.\ with common law $\mathcal L_{\tau, N}^{\dagger}$ and CF $G_N^{\dagger}$. Therefore,
\[
\mathbb E^{\dagger}\!\left[e^{i\eta X_m^{\dagger}}\right]=G_N^{\dagger}(\eta),
\qquad
\mathbb E^{\dagger}\!\left[\big|e^{i\eta X_m^{\dagger}}\big|^2\right]=1.
\]
The first equation of \eqref{eq:boot_mean} follows from \eqref{eq:hatG_boot}. Using the same variance calculation as in \eqref{eq:ecf_mean_var_section2} gives the second equation of \eqref{eq:boot_mean}. The identities \eqref{eq:boot_mse_mae}--\eqref{eq:boot_loss_bound} then follow by repeating the proof of Lemma~\ref{lem:ecf_error_section3} with $G$ replaced by $G_N^{\dagger}$ and $\widehat G_M$ replaced by $\widehat G_M^{\dagger}$.

{\dpurple{
\section{Square integrability of the multivariate Laplace kernel}
\label{app:multi_laplace_L2}
For the multivariate Laplace kernel in \eqref{eq:multi_laplace_kernel}, the squared modulus of its CF is
\[
\left|
\frac{\exp(i\boldsymbol{\eta}^{\top}\boldsymbol{\nu})}
{1+\boldsymbol{\eta}^{\top}\mathbf{\Lambda}\boldsymbol{\eta}}
\right|^2
=
\frac{1}{\bigl(1+\boldsymbol{\eta}^{\top}\mathbf{\Lambda}\boldsymbol{\eta}\bigr)^2}.
\]
Since $\mathbf{\Lambda}$ is symmetric positive definite, the change of variables
$\mathbf{y}=\mathbf{\Lambda}^{1/2}\boldsymbol{\eta}$ gives
\[
\int_{\Rbb^d}
\frac{d\boldsymbol{\eta}}
{\bigl(1+\boldsymbol{\eta}^{\top}\mathbf{\Lambda}\boldsymbol{\eta}\bigr)^2}
=
|\mathbf{\Lambda}|^{-1/2}
|\mathbb S^{d-1}|
\int_0^\infty \frac{r^{d-1}}{(1+r^2)^2}\,dr.
\]
The radial integrand behaves as $r^{d-1}$ near the origin and as $r^{d-5}$ as $r\to\infty$. Hence the integral is finite if and only if $d<4$. By Plancherel's theorem, the multivariate Laplace kernel belongs to $L_2(\Rbb^d)$ precisely for $d=1,2,3$.

Consequently, if at least one Laplace component has positive weight, the positive Gaussian--Laplace mixture in \eqref{eq:Sigma_GL_multi} is not in $L_2(\Rbb^d)$ for $d\ge4$, since
$\widehat g(\mathbf{x};\theta)\ge
\beta_\ell^{(\mathcal L)}\varphi_{\mathcal L}^{(d)}(\mathbf{x};\boldsymbol{\nu}_\ell,\mathbf{\Lambda}_\ell)$.
Thus the restriction to $d=1,2,3$ is necessary for the present $L_2$ analysis whenever a Laplace component has positive weight. The Gaussian-only subclass remains in $L_2(\Rbb^d)$ for every $d\ge1$.
}}

\vfill
\end{document}